\def\draft{0} 

\newif\ifdraft
 \draftfalse

\documentclass[11pt]{article}

\usepackage[letterpaper,margin=1in]{geometry}
\usepackage[parfill]{parskip}

\usepackage{authblk}

\usepackage[toc,page,header]{appendix}

\usepackage[utf8]{inputenc} 
\usepackage[T1]{fontenc}    
\usepackage[colorlinks,citecolor=blue,urlcolor=blue,linkcolor=blue,linktocpage=true]{hyperref}       
\usepackage{url}            
\usepackage{booktabs}       
\usepackage{amsfonts}       
\usepackage{nicefrac}       
\usepackage{microtype}      
\usepackage{xcolor}         
\label{key}
\usepackage{thmtools}

\usepackage{comment}
\usepackage{amsmath}

\usepackage{epsfig}
\usepackage{graphicx}
\usepackage{wrapfig}
\usepackage{subfigure}
\usepackage{multirow}
\usepackage{hyperref}
\usepackage{graphicx}
\usepackage{caption}
\usepackage{array}
\usepackage{bbm}
\usepackage{color}
\usepackage{enumerate}
\usepackage{enumitem}
\setlist{leftmargin=10mm}
\usepackage{mathtools}
\usepackage{amsmath,amssymb,amsthm,bm}
\usepackage{amsfonts,graphicx}
\usepackage{mathrsfs}
\usepackage{algorithmic,algorithm}

\usepackage{xr}

\usepackage[authoryear]{natbib}

\def\draft{0} 

\newif\ifdraft
 \draftfalse
 
\ifnum\draft=1
\newcommand{\add}[1]{\textcolor{purple}{#1}}
\newcommand{\revise}[1]{\textcolor{red}{#1}}
\else
\newcommand{\add}[1]{#1}
\newcommand{\revise}[1]{#1}
\fi

\newenvironment{packeditemize}{
\begin{list}{$\bullet$}{
\setlength{\labelwidth}{8pt}
\setlength{\itemsep}{0pt}
\setlength{\leftmargin}{\labelwidth}
\addtolength{\leftmargin}{\labelsep}
\setlength{\parindent}{0pt}
\setlength{\listparindent}{\parindent}
\setlength{\parsep}{0pt}
\setlength{\topsep}{3pt}}}{\end{list}}

\theoremstyle{plain}
\newtheorem{theorem}{Theorem}[section]

\newtheorem{lemma}[theorem]{Lemma}

\theoremstyle{definition}
\newtheorem{definition}[theorem]{Definition}

\theoremstyle{remark}
\newtheorem{remark}[theorem]{Remark}

\newtheorem{innercustomthm}{Theorem}
\newenvironment{customthm}[1]
  {\renewcommand\theinnercustomthm{#1}\innercustomthm}
  {\endinnercustomthm}

\newcommand{\E}{\mathbb{E}}
\newcommand{\Var}{\mathrm{Var}}
\newcommand{\R}{\mathbb{R}}
\newcommand{\eps}{\varepsilon}
\newcommand{\D}{\mathcal{D}}
\newcommand{\norm}[1]{\left\lVert#1\right\rVert}
\newcommand{\N}{\mathcal{N}}
\newcommand{\I}{\mathcal{I}}

\DeclareMathOperator*{\argmin}{\arg\!\min}

\newcommand{\g}{\nabla}
\newcommand{\iden}{\bm{1}}

\newcommand{\mS}{\mathcal{S}}
\newcommand{\own}{\ni}
\newcommand{\notown}{\not\own}

\newcommand{\A}{\mathcal{A}}
\newcommand{\metric}{\texttt{acc}}
\newcommand{\std}{\texttt{Std}}

\newcommand{\unif}{\mathrm{Unif}}
\newcommand{\semi}{\mathrm{semi}}
\newcommand{\shap}{\mathrm{shap}}
\newcommand{\loo}{\mathrm{loo}}
\newcommand{\banz}{\mathrm{banz}}
\newcommand{\mc}{\mathrm{MC}}
\newcommand{\gt}{\mathrm{MSR}}
\newcommand{\Sowni}{\mS_{\own i}}
\newcommand{\Snotowni}{\mS_{\notown i}}
\newcommand{\bin}{\mathrm{Bin}}
\newcommand{\Uhat}{\widehat{U}}
\newcommand{\U}{\mathcal{U}}
\newcommand{\Utau}{\U_{i, j}^{(\tau)}}
\newcommand{\safe}{\texttt{Safe}}
\newcommand{\tildephi}{\widetilde \phi}
\newcommand{\tildeO}{\widetilde{O}}

\newcommand{\lc}{\mathrm{LC}}

\usepackage{dsfont}
\newcommand{\ind}{\mathds{1}}

\usepackage{cleveref}

\author[1]{Jiachen T. Wang}
\author[2]{Ruoxi Jia}

\affil[1]{Princeton University}
\affil[2]{Virginia Tech\protect\\
\texttt{\small tianhaowang@princeton.edu}, 
\texttt{\small ruoxijia@vt.edu}
}

\date{}

\title{Data Banzhaf: A Robust Data Valuation Framework for Machine Learning}

\begin{document}



\maketitle

\begin{abstract}
Data valuation has wide use cases in machine learning, including improving data quality and creating economic incentives for data sharing. This paper studies the robustness of data valuation to noisy model performance scores. Particularly, we find that the inherent randomness of the widely used stochastic gradient descent can cause existing data value notions (e.g., the Shapley value and the Leave-one-out error) to produce inconsistent data value rankings across different runs. To address this challenge, we introduce the concept of safety margin, which measures the robustness of a data value notion. We show that the Banzhaf value, a famous value notion that originated from cooperative game theory literature, achieves the largest safety margin among all semivalues (a class of value notions that satisfy crucial properties entailed by ML applications and include the famous Shapley value and Leave-one-out error). We propose an algorithm to efficiently estimate the Banzhaf value based on the Maximum Sample Reuse (MSR) principle. Our evaluation demonstrates that the Banzhaf value outperforms the existing semivalue-based data value notions on several ML tasks such as learning with weighted samples and noisy label detection. Overall, our study suggests that when the underlying ML algorithm is stochastic, the Banzhaf value is a promising alternative to the other semivalue-based data value schemes given its computational advantage and ability to robustly differentiate data quality.\footnote{Code is available at \url{https://github.com/Jiachen-T-Wang/data-banzhaf}.} 
\end{abstract}

\section{ Introduction }
\label{sec:intro}


\revise{
Data valuation, i.e., quantifying the usefulness of a data source, is an essential component in developing machine learning (ML) applications. For instance, evaluating the worth of data plays a vital role in cleaning bad data \citep{tang2021data, karlavs2022data} and understanding the model's test-time behavior \citep{koh2017understanding}. Furthermore, determining the value of data is crucial in creating incentives for data sharing and in implementing policies regarding the monetization of personal data \citep{ghorbani2019data, zhu2019incentive}. 
}


\revise{Due to the great potential in real applications,}
there has been a surge of research efforts on developing data value notions for supervised ML~\citep{jia2019towards,ghorbani2019data,yan2020ifyoulike,ghorbani2021data,kwon2021beta,yoon2020data}. 
In the ML context, a data point's value 
depends on other data points used in model training. 
For instance, a data point's value will decrease if we add extra data points that are similar to the existing one into the training set. 
To accommodate this interplay, current data valuation techniques typically start by defining the ``utility'' of a \emph{set} of data points, and then measure the value of an \emph{individual} data point based on the change of utility when the point is added to an existing dataset. 
For ML tasks, the utility of a dataset is naturally chosen to be the performance score (e.g., test accuracy) of a model trained on the dataset. 

\revise{However, the utility scores can be noisy and unreliable.} 
Stochastic training methods such as stochastic gradient descent (SGD) are widely adopted in ML, especially for deep learning. The models trained with stochastic methods are inherently random, and so are their performance scores. This, in turn, makes the data values calculated from the performance scores \emph{noisy}. Despite being ignored in past research, we find that the noise in a typical learning process is actually substantial enough to make different runs of the same data valuation algorithm produce inconsistent value rankings. Such inconsistency can pose challenges for building reliable applications based on the data value scores and rankings, e.g., low-quality data identification.

In this paper, we initiate the study of the robustness of data valuation to noisy model performance scores. Our technical contributions are listed as follows.


\underline{\textbf{Theoretical framework for quantifying robustness.}} 
We start by formalizing what it means mathematically for a data value notion to be robust. We introduce the concept of \emph{safety margin}, which is the magnitude of the largest perturbation of model performance scores that can be tolerated so that the value order of every pair of data points remains unchanged. We consider the two most popular data valuation schemes---the Shapley value and the Leave-one-out (LOO) error and show that the safety margin of the Shapley value is greater than that of the LOO error. Our results shed light on a common observation in the past works~\citep{ghorbani2019data,jia2019scalability} that the Shapley value often outperforms the LOO error in identifying low-quality training data.


\underline{\textbf{Banzhaf value: a robust data value notion.}}
Surprisingly, we found that the Banzhaf value \citep{banzhaf1964weighted}, a classic value notion from cooperative game theory that was proposed more than half a century ago, achieves the largest safety margin among all semivalues --- a collection of value notions (including LOO error and the Shapley value) that satisfy essential properties of a proper data value notion in the ML context~\citep{kwon2021beta}. 
Particularly, the safety margin of the Banzhaf value is exponentially larger than that of the Shapley value and the LOO error. 

\underline{\textbf{Efficient Banzhaf value estimation algorithm.}}
Similar to the Shapley value, the Banzhaf value is also costly in computation. 
We present an efficient estimation algorithm based on the Maximum Sample Reuse (MSR) principle, which can achieve $\ell_\infty$ and $\ell_2$ error guarantees for approximating the Banzhaf value with logarithmic and nearly linear sample complexity, respectively. 
We show that the existence of an efficient MSR estimator is \emph{unique} for the Banzhaf value among all existing semivalue-based data value notions. 
\revise{
We derive a lower bound of sample complexity for the Banzhaf value estimation, and show that our MSR estimator's sample complexity \emph{is close to}
this lower bound.} 
Additionally, we show that the MSR estimator is robust against the noise in performance scores. 

\underline{\textbf{Empirical evaluations.}} 
Our evaluation demonstrates the ability of the Banzhaf value in preserving value rankings given noisy model performance scores. We also empirically validate the sample efficiency of the MSR estimator for the Banzhaf value. We show that the Banzhaf value outperforms the state-of-the-art semivalue-based data value notions (including the Shapley value, the LOO error, and the recently proposed Beta Shapley~\citep{kwon2021beta}) on several ML tasks including bad data detection and data reweighting, when the underlying learning algorithm is SGD.

We call the suite of our data value notion and the associated estimation algorithm as the \emph{Data Banzhaf} framework. 
Overall, our work suggests that Data Banzhaf is a promising alternative to the existing semivalue-based data value notions given its computational advantage and the ability to robustly distinguish data quality in the presence of learning stochasticity. 
\section{ Background: From Leave-One-Out to Shapley to Semivalue }
\label{sec:background}

In this section, we formalize the data valuation problem for ML. Then, we review the concept of LOO and Shapley value---the most popular data value notions in the existing literature, as well as the framework of \emph{semivalues}, which are recently introduced as a natural relaxation of Shapley value in the ML context.


\textbf{Data Valuation Problem Set-up.} Let $N=\{1,\ldots,n\}$ denotes a training set of size $n$. The objective of data valuation is to assign a score to each training data point in a way that reflects their contribution to model training. 
We will refer to these scores as \emph{data values}. 
To analyze a point's ``contribution'', we define a \emph{utility function} $U: 2^N \rightarrow \R$, which maps any subset of the training set to a score indicating the usefulness of the subset. $2^N$ represents the power set of $N$, i.e., the collection of all subsets of $N$, including the empty set and $N$ itself. For classification tasks, a common choice for $U$ is the validation accuracy of a model trained on the input subset. Formally, we have $U(S) = \metric(\A(S))$, where $\A$ is a learning algorithm that takes a dataset $S$ as input and returns a model, and $\metric$ is a metric function to evaluate the performance of a given model, e.g., the accuracy of a model on a hold-out test set. Without loss of generality, we assume throughout the paper that $U(S) \in [0, 1]$. For notational simplicity, we sometimes denote $S \cup i := S \cup \{i\}$ and $S \setminus i := S \setminus \{i\}$ for singleton $\{i\}$, where $i \in N$ represents a single data point. 

We denote the data value of data point $i \in N$ computed from $U$ as $\phi(i; U)$. We review the famous data value notions in the following.


\textbf{LOO Error.} A simple data value measure is leave-one-out (LOO) error, which calculates the change of model performance when the data point $i$ is excluded from the training set $N$:
\begin{align}
    \phi_\text{loo}(i;U) := U(N)-U(N\setminus i)
\end{align}
However, many empirical studies~\citep{ghorbani2019data,jia2019scalability} suggest that it underperforms other alternatives in differentiating data quality.

\textbf{Shapley Value.} The Shapley value is arguably the most widely studied scheme for data valuation. 
At a high level, it appraises each point based on the (weighted) average utility change caused by adding the point into different subsets. The Shapley value of a data point $i$ is defined as 

{\small
\begin{align}
&\phi_{\shap}\left(i;U\right) \nonumber \\
&:= \frac{1}{n} \sum_{k=1}^{n} {n-1 \choose k-1}^{-1} \sum_{S \subseteq N \setminus \{i\}, |S|=k-1} \left[ U(S \cup i) - U(S) \right] \nonumber
\end{align}
}

The popularity of the Shapley value is attributable to the fact that it is the \emph{unique} data value notion satisfying the following four axioms~\citep{shapley1953value}:
\begin{packeditemize}
    \item Dummy player: if $U\left(S \cup i\right)=U(S)+c$ for all $S \subseteq N \setminus i$ and some $c \in \mathbb{R}$, then $\phi\left(i ; U\right)=c$.
    \item Symmetry: if $U(S \cup i) = U(S \cup j)$ for all $S \subseteq N \setminus \{i, j\}$, then $\phi(i; U)=\phi(j; U)$. 
        \item Linearity: For utility functions $U_1, U_2$ and any $\alpha_1, \alpha_2 \in \R$, $\phi \left(i; \alpha_{1} U_{1}+\alpha_{2} U_{2}\right)=\alpha_{1} \phi\left(i ; U_{1}\right)+$ $\alpha_{2} \phi\left(i ; U_{2}\right)$.
    \item Efficiency: for every $U, \sum_{i \in N} \phi(i ; U)=U(N)$.
\end{packeditemize}
The difference $U(S \cup i) - U(S)$ is often termed the \emph{marginal contribution} of data point $i$ to subset $S \subseteq N \setminus i$. 
We refer the readers to \citep{ghorbani2019data,jia2019towards} for a detailed discussion about the interpretation of dummy player, symmetry, and linearity axioms in ML. 
The \emph{efficiency} axiom, however, receives more controversy than the other three. 
The efficiency axiom requires the total sum of data values to be equal to the utility of full dataset $U(N)$. 
Recent work~\citep{kwon2021beta} argues that this axiom is considered not essential in ML. Firstly, the choice of utility function in the ML context is often not directly related to monetary value so it is unnecessary to ensure the sum of data values matches the total utility. Moreover, many applications of data valuation, such as bad data detection, are performed based only on the ranking of data values. For instance, multiplying the Shapley value by a positive constant does not affect the ranking of the data values. Hence, there are many data values that do not satisfy the efficiency axiom, but can still be used for differentiating data quality, just like the Shapley value.


\textbf{Semivalue.} The class of data values that satisfy all the Shapley axioms except efficiency is called \emph{semivalues}. It was originally studied in the field of economics and recently proposed to tackle the data valuation problem~\citep{kwon2021beta}.
Unlike the Shapley value, semivalues are \emph{not} unique. The following theorem by the seminal work of \citep{dubey1981value} shows that every semivalue of a player $i$ (in our case the player is a data point) can be expressed as the weighted average of marginal contributions $U(S \cup i)-U(S)$ across different subsets $S \subseteq N \setminus i$. 

\begin{theorem}[Representation of Semivalue \citep{dubey1981value}]
A value function $\phi_{\semi}$ is a semivalue, if and only if, there exists a \emph{weight function} $w:[n] \rightarrow \mathbb{R}$ such that $\sum_{k=1}^{n}{n-1 \choose k-1} w(k)=n$ and the value function $\phi_{\semi}$ can be expressed as follows:

{\small
\begin{equation}
\begin{split}
\phi_{\semi}\left(i ; U, w\right) := \sum_{k=1}^{n} \frac{w(k)}{n} \sum_{\substack{S \subseteq N \setminus \{i\},\\ |S|=k-1}} \left[ U(S \cup i) - U(S) \right] \label{eq:semivalue}
\end{split}
\end{equation}
}
\end{theorem}

Semivalues subsume both the Shapley value and the LOO error with $w_{\shap}(k) = {n-1 \choose k-1}^{-1}$ and $w_{\loo}(k) = n \iden[k=n]$, respectively.
Despite the theoretical attraction, the question remains which one of the many semivalues we should adopt. 

\section{Utility Functions Can Be Stochastic }
\label{sec:utility_stochastic}

\newcommand{\width}{1}




In the existing literature, the utility of a dataset $U(S)$ is often defined to be $ \metric(\A(S))$, i.e., the performance of a model $\A(S)$ trained on a dataset $S$. However, many learning algorithms $\A$ such as SGD contain randomness. Since the loss function for training neural networks is non-convex, the trained model depends on the randomness of the training process, e.g., random mini-batch selection. Thus, $U(S)$ defined in this way inherently becomes a randomized function. 
\add{As noted in many studies on the reproducibility of neural network training, the learning stochasticity can introduce large variations into the predictive performance of deep learning models~\citep{summers2021nondeterminism, zhuang2022randomness, raste2022quantifying}. 
On the other hand, the existing data value notions compute the value of data points based on the performance scores of models trained on different data subsets and therefore will also be noisy given stochastic learning algorithms.
In this section, we delve into the influence of learning stochasticity on data valuation results, and show that the run-to-run variability of the resulting data value rankings is large for the existing data value notions.}

\begin{figure*}[t]
    \centering
    \includegraphics[width=\textwidth]{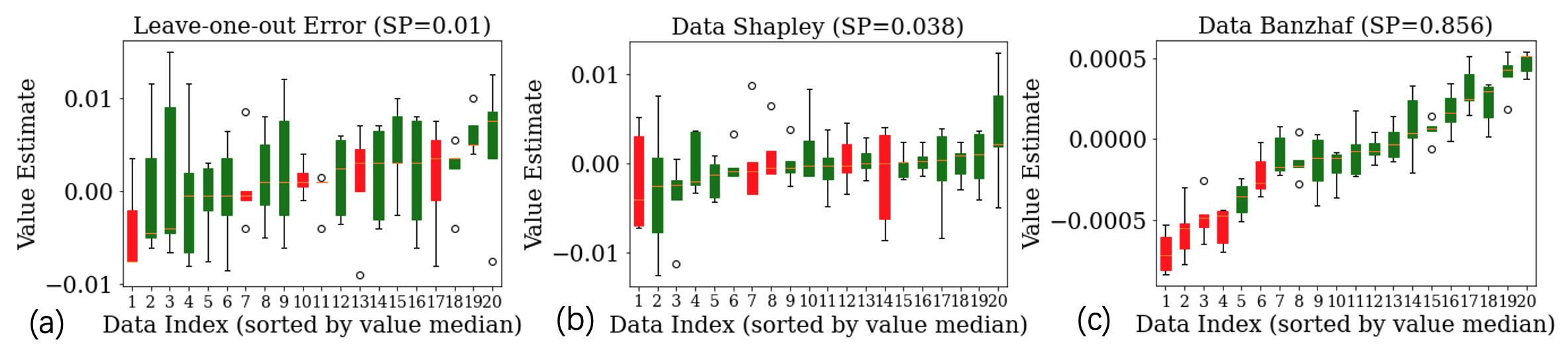}
    \caption{Box-plot of the estimates of (a) LOO, (b) Shapley Value, and (c) Banzhaf value of 20 randomly selected CIFAR10 images, with 5 mislabeled images. 
    The 5 mislabeled images are shown in red and clean images are shown in green. 
    The variance is \emph{only} due to the stochasticity in utility evaluation. 
    `SP' means the average Spearman index across different runs. }
    \label{fig:checkvar-value}
\end{figure*}

\begin{figure}[h]
    \centering
    \includegraphics[width=\columnwidth]{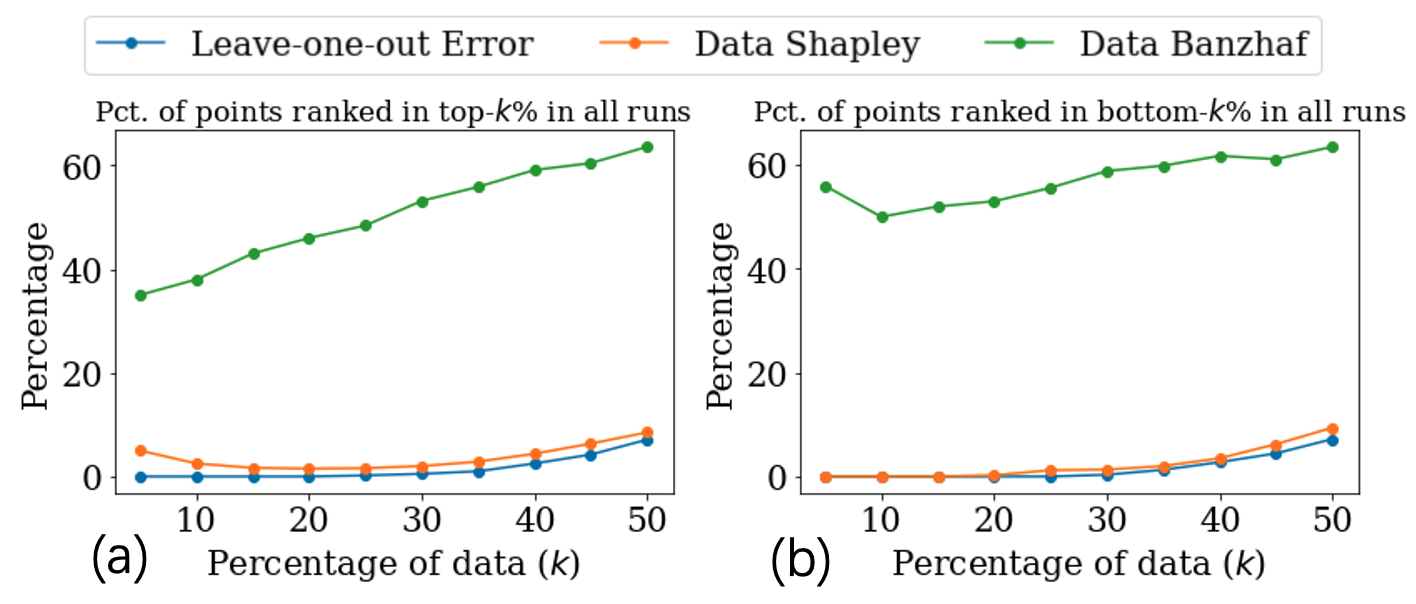}
    \caption{
    The percentage of the CIFAR10 data points that are ranked in (a) top-$k\%$ and (b) bottom-$k\%$ across all runs, among the top/bottom-$k\%$ data points.  
    }
    \label{fig:topk}
\end{figure}


\textbf{Instability of data value rankings.} 
\add{Semivalues are calculated by taking a weighted average of marginal contributions. When the weights are not properly chosen, the noisy estimate of marginal contributions can cause significant instability in ranking the data values. Figure \ref{fig:checkvar-value} (a)-(b) illustrate the distribution of the estimates of two popular data value notions---LOO error and the Shapley value---across 5 runs with different training random seeds. 
The utility function is the accuracy of a neural network trained via SGD on a held-out dataset; we show the box-plot of the estimates' distribution for 20 CIFAR10 images, with 5 of them being mislabeled (marked in red). The experiment settings are detailed in Appendix \ref{appendix:experiment-figuretwo}. As we can see, the variance of the data value estimates caused by learning stochasticity significantly outweighs their magnitude for both LOO and the Shapley value. As a result, the rankings of data values across different runs are largely inconsistent (the average Spearman coefficient of individual points' values across different runs for LOO is $\approx 0.001$ and for Shapley is $\approx 0.038$). Leveraging the rankings of such data values to differentiate data quality is unreliable, as we can see that the 5 mislabeled images' value estimates distribution has a large overlap with the value estimates of the clean images. Further investigation of these data value notion's efficacy in identifying data quality is provided in the Evaluation section. }

When interpreting a learning algorithm, one may be interested in finding a small set of data points with the most positive/negative influences on model performance. In Figure \ref{fig:topk}, we show how many data points are consistently ranked in the top or bottom-$k\%$ across all the runs. Both LOO and the Shapley value has only $<10\%$ data points that are consistently ranked high/low for any $k \le 50\%$. This means that the high/low-influence data points selected by these data value notions have a large run-to-run variation and cannot form a reliable explanation for model behaviors.


\textbf{Redefine $U$ as expected performance.}
To make the data value notions independent of the learning stochasticity, a natural choice is to redefine $U$ to be $U(S):=\E_{\A}[\metric(\A(S))]$, i.e., the \emph{expected performance} of the trained model. However, accurately estimating $U(S)$ under this new definition requires running $\A$ multiple times on the same $S$, and calculating the average utility of $S$. Obviously, this simple approach incurs a large extra computational cost. On the other hand, if we estimate $U(S)$ with only one or few calls of $\A$, the estimate of $U(S)$ will be very noisy.
Hence, we pose the question: 
\emph{how to find a more robust semivalue against perturbation in model performance scores?}




\section{Data Banzhaf: a Robust Data Value Notion}
\label{sec:databanzhaf}

To address the question posed above, this section starts by formalizing the notion of robustness in data valuation. Then, we show that the most robust semivalue, surprisingly, coincides with the Banzhaf value \citep{banzhaf1964weighted}---a classic solution concept in cooperative game theory. 


\subsection{Ranking Stability as a Robustness Notion}
\label{sec:safety-margin-def}

In many applications of data valuation such as data selection, it is the \textbf{order} of data values that matter~\citep{kwon2021beta}. For instance, to filter out low-quality data, one will first rank the data points based on their values and then throws the points with the lowest values. 
When the utility functions are perturbed by noise, we would like the rankings of the data values to remain stable. 
Recall that a semivalue is defined by a weight function $w$ such that $\sum_{k=1}^{n}{n-1 \choose k-1} w(k)=n$. 
The (scaled) difference between the semivalues of two data points $i$ and $j$ can be computed from (\ref{eq:semivalue}):
\begin{equation}
\begin{split}
    D_{i, j}(U; w) 
    &:= n( \phi(i; w) - \phi(j; w) ) \\
    &= \sum_{k=1}^{n-1} \left(w(k) + w(k+1)\right) {n-2 \choose k-1} \Delta_{i, j}^{(k)}(U),
    \nonumber 
\end{split}
\end{equation}
where $\Delta_{i, j}^{(k)}(U) := {n-2 \choose k-1}^{-1} \sum_{|S|=k-1, S \subseteq N \setminus \{i, j\}} [U(S \cup i) - U(S \cup j)]$, representing the \textit{average distinguishability between $i$ and $j$ on size-$k$ sets using the noiseless utility function $U$.}
Let $\Uhat$ denote a noisy estimate of $U$. 
We can see that $\Uhat$ and $U$ produce different data value orders for $i, j$ if and only if $D_{i, j}(U; w) D_{i, j}(\Uhat; w) \le 0$.\footnote{We note that when the two data points receive the same value, we usually break tie randomly, thus we use $\le 0$ instead of $<0$.} 
An initial attempt to define robustness is in terms of the smallest amount of perturbation magnitude $\norm{\Uhat - U}$ such that $U$ and $\Uhat$ produce different data rankings.\footnote{Here, we view the utility function $U$ as a size-$2^n$ vector where each entry corresponds to $U(S)$ of a subset $S \subseteq N$.}
However, such a definition is problematic due to its dependency on the original utility function $U$.
If the noiseless $U$ itself cannot sufficiently differentiate between $i$ and $j$ 
(i.e., $\Delta_{i, j}^{(k)}(U)\simeq 0$ for $k=1,\ldots,n-1$), then
$D_{i, j}(U; w)$ will be (nearly) zero and infinitesimal perturbation can switch the ranking of $\phi(i)$ and $\phi(j)$. 
To reasonably define the robustness of semivalues, we solely consider the collection of utility functions that can sufficiently ``distinguish'' between $i$ and $j$.
\begin{definition}[Distinguishability]
We say a data point pair $(i, j)$ is \emph{$\tau$-distinguishable} by $U$ if and only if 
$\Delta_{i, j}^{(k)}(U) \ge \tau$ for all $k \in \{1, \ldots, n-1\}$. 
\label{def:distinguishable}
\end{definition}

Let $\Utau$ denote the collection of utility functions $U$ that can $\tau$-distinguish a pair $(i, j)$. With the definition of distinguishability, we characterize the robustness of a semivalue by deriving its ``\emph{safety margin}'', which is defined as the \textbf{minimum} amount of perturbation $\norm{\Uhat - U}$ needed to reverse the ranking of \textbf{at least one} pair of data points $(i, j)$, for \textbf{at least one} utility function $U$ from $\Utau$.

\begin{definition}[Safety margin]
\label{def:safety-margin}
Given $\tau>0$, we define the safety margin of a semivalue for a data point pair $(i, j)$ as

{\small 
\begin{align}
    &\safe_{i, j}(\tau; w) := \min_{U \in \Utau} \min_{\Uhat \in \{\Uhat: D_{i, j}(U; w) D_{i, j}(\Uhat; w) \le 0\}} \norm{\Uhat - U} \nonumber
\end{align}
}
and we define the safety margin of a semivalue as 

{\small 
\begin{align}
    \safe(\tau; w) := \min_{i, j \in N, i \ne j} \safe_{i, j}(\tau; w) \nonumber
\end{align}
}
\end{definition}

\begin{figure}[t]
    \centering
    \includegraphics[width=0.5\columnwidth]{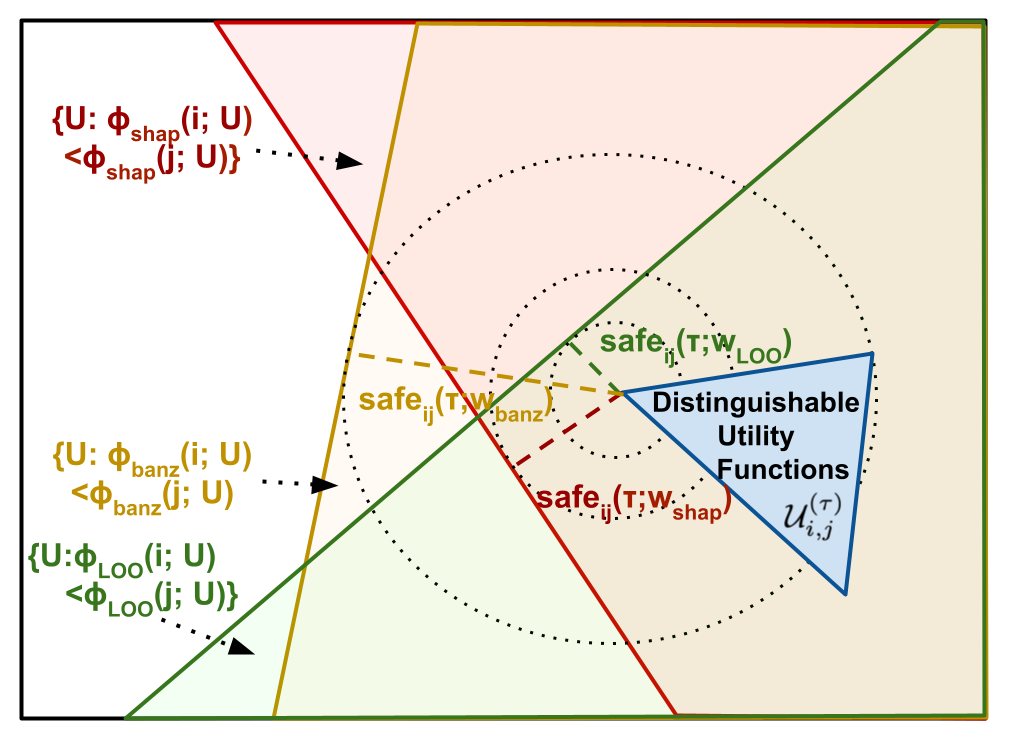}
    \caption{Illustration of Safety Margin. Intuitively, it is the smallest distance from the space of distinguishable utility function to the space of functions that will reverse the value rank between at least one pair of data points $i$ and $j$. The figure shows the distance between the space of distinguishable utility function and the space of utility function where $\phi_\banz(i; U) < \phi_\banz(j; U)$ (i.e., the order of $i, j$ being reversed) is the largest among the three value notions.
    }
    \label{fig:banzhaf_geo}
\end{figure}

In other words, the safety margin captures the largest noise that can be tolerated by a semivalue without altering the ranking of any pair of data points that are distinguishable by the original utility function. The geometric intuition of safety margin is illustrated in Figure~\ref{fig:banzhaf_geo}.
\add{
\begin{remark}
\label{remark:why-noise-agonostic}
\textbf{The definition of the safety margin is \emph{noise-structure-agnostic} in the sense that it does not depend on the actual noise distribution induced by a specific stochastic training algorithm.} 
While one might be tempted to have a noise-dependent robustness definition, we argue that the safety margin is advantageous from the following aspects: \textbf{(1) The analysis of the utility noise distribution caused by stochastic training is difficult even for very simple settings.} In Appendix \ref{sec:dist-in-simple-setting}, we consider a simple (if not the simplest) setting: 1-dimensional linear regression trained by full-batch gradient descent with Gaussian random initialization. We show that, even for such a setting, there are significant technical challenges in making any analytical assertion on the noise distribution. 
\textbf{(2) It might be computationally infeasible to estimate the perturbation distribution of $U(S)$,} as there are exponentially many $S$ that need to be considered. 
\textbf{(3) One may not have prior knowledge about the dataset and source of noise; in that case, a worst-case robustness notion is preferred.} 
In practice, the datasets may not be known in advance for privacy consideration \citep{agahari2022not}. 
Furthermore, the performance scores may also be perturbed due to other factors such as hardware faults, software bugs, or even adversarial attacks. Given the diversity and unknownness of utility perturbations in practical scenarios, it is preferable to define the robustness in terms of the worst-case perturbation (i.e., Kerckhoffs's principle). Such definitions are common in machine learning. For instance, robust learning against adversarial examples is aimed at being resilient to the worst-case noise \citep{madry2017towards} within a norm ball; differential privacy \citep{dwork2006calibrating} protects individual data record's information from arbitrary attackers.
\end{remark}
}
\add{A full discussion for important considerations in Definition \ref{def:safety-margin} can be found in Appendix \ref{appendix:robustness-discussion}.}

\add{\textbf{Safety Margin for LOO and Shapley Value.} In order to demonstrate the usefulness of this robustness notion, we derive the LOO and Shapley value's safety margin.}

\begin{theorem}
For any $\tau>0$, Leave-one-out error ($w_{\loo}(k) = n \iden[k=n]$) achieves $\safe(\tau; w_{\loo})=\tau$, and Shapley value ($w_{\shap}(k) = {n-1 \choose k-1}^{-1}$) achieves $\safe(\tau; w_{\shap}) = \tau \frac{n-1}{\sqrt{\sum_{k=1}^{n-1} {n-2 \choose k-1}^{-1} }}$. 
\label{thm:safe-loo-shap}
\end{theorem}

One can easily see that $\safe(\tau; w_{\loo})<\safe(\tau; w_{\shap})<\tau(n-1)$. 
The fact that $\safe(\tau; w_{\loo}) < \safe(\tau; w_{\shap})$ sheds light on the phenomenon we observe in Figure \ref{fig:checkvar-value} and \ref{fig:topk} where the Shapley value is slightly more stable than LOO. 
\revise{
It provides an explanation for a widely observed but puzzling phenomenon observed in several prior works \citep{jia2019scalability, wang2020principled} that the Shapley value outperforms the LOO error in a range of data selection tasks in the stochastic learning setting. We note that the Shapley value is also shown to be better than LOO in deterministic learning \citep{ghorbani2019data, jia2019towards}, where theoretical underpinning is an open question.}

\subsection{Banzhaf Value Achieves the Largest Safety Margin}

Surprisingly, the semivalue that achieves the largest safety margin coincides with the Banzhaf value, another famous value notion that averages the marginal contribution across all subsets. We first recall the definition of the Banzhaf value.

\begin{definition}[\cite{banzhaf1964weighted}]
The Banzhaf value for data point $i$ is defined as 
\begin{align}
    \phi_{\banz}(i; U, N) 
    := \frac{1}{2^{n-1}} \sum_{S \subseteq N \setminus i} [ U(S \cup i) - U(S) ] \label{eq:banzhaf}
\end{align}
\label{def:banzhaf}
\end{definition}

The Banzhaf value is a semivalue, as we can recover its definition (\ref{eq:banzhaf}) from the general expression of semivalues (\ref{eq:semivalue}) by setting the constant weight function $w(k)=\frac{n}{2^{n-1}}$ for all $k \in \{1, \ldots, n\}$. We then show our main result. 
\begin{theorem}
For any $\tau>0$, Banzhaf value ($w(k)=\frac{n}{2^{n-1}}$) achieves the largest safety margin $\safe(\tau; w)=\tau 2^{n/2-1}$ among all semivalues. 
\label{thm:banzhaf-robustness}
\end{theorem}


\textbf{Intuition.} 
The superior robustness of the Banzhaf value can be explained intuitively as follows: Semivalues assign different weights to the marginal contribution against different data subsets according to the weight function $w$. To construct a perturbation of the utility function that maximizes the influence on the corresponding semivalue, one needs to perturb the utility of the subsets that are assigned with higher weights. 
Hence, the best robustification strategy is to assign uniform weights to all subsets, 
which leads to the Banzhaf value. On the other hand, semivalues that assign heterogeneous weights to different subsets, such as the Shapley value and LOO error, suffer a lower safety margin.

\begin{remark}
\textbf{Banzhaf value is also the most robust semivalue in terms of the data value magnitude.}
One can also show that the Banzhaf value is the most robust in the sense that the utility noise will minimally affect data value magnitude changes. Specifically, the Banzhaf value achieves the smallest Lipschitz constant $L$ such that $\| \phi(U) - \phi(\widehat U) \| \le L \| U - \widehat U \|$ for all possible pairs of $U$ and $\widehat U$. The details are deferred to Appendix \ref{sec:l2-robustness}.
\end{remark}

\subsection{Efficient Banzhaf Value Estimation}
\label{sec:banzhaf-computation}


Similar to the Shapley value and the other semivalue-based data value notions, the exact computation of the Banzhaf value can be expensive because it requires an exponential number of utility function evaluations, which entails an exponential number of model fittings. This could be a major challenge for adopting the Banzhaf value in practice. To address this issue, we present a novel Monte Carlo algorithm to approximate the Banzhaf value. 

We start by defining the estimation error of an estimator. We say a semivalue estimator $\widehat \phi$ is an $(\eps, \delta)$-approximation to the true semivalue $\phi$ (in $\ell_p$-norm) if and only if
$
\Pr_{\widehat \phi} [ \| \phi - \widehat \phi \|_p \le \eps ] \ge 1-\delta
$
where the randomness is over the execution of the estimator. For any data point pair $(i, j)$, if $|\phi(i)-\phi(j)| \ge 2\eps$, then an estimator that is $(\eps, \delta)$-approximation in $\ell_\infty$-norm is guaranteed to keep the data value order of $i$ and $j$ with probability at least $1-\delta$. 

\textbf{Baseline: Simple Monte Carlo.} The Banzhaf value can be equivalently expressed as follows:
\begin{align}
    \phi_\banz(i) = \E_{S \sim \unif(2^{N \setminus i})} \left[  U(S \cup i) - U(S) \right]
\end{align}
where $\unif(\cdot)$ to denote Uniform distribution over the power set of $N\setminus\{i\}$.
Thus, a straightforward Monte Carlo (MC) method to estimate $\phi_\banz(i)$ is to sample a collection of data subsets $\mS_i$ from $2^{N \setminus i}$ uniformly at random, and then compute $\widehat \phi_\mc(i) = \frac{1}{|\mS_i|} \sum_{S \in \mS_i} \left( U(S \cup i) - U(S) \right)$. 
We can repeat the above procedure for each $i \in N$ and obtain the approximated semivalue vector $\widehat \phi_\mc = [\widehat \phi_\mc(1), \ldots, \widehat \phi_\mc(n)]$. 
The sample complexity of this simple MC estimator can be bounded by Hoeffding's inequality. 


\begin{theorem}
$\widehat \phi_\mc$ is an $(\eps, \delta)$-approximation to the exact Banzhaf value in $\ell_2$-norm with $O(\frac{n^2}{\eps^2}\log(\frac{n}{\delta}))$ calls of $U(\cdot)$, and in $\ell_\infty$-norm with $O(\frac{n}{\eps^2}\log(\frac{n}{\delta}))$ calls of $U(\cdot)$. 
\label{thm:mc}
\end{theorem}

\textbf{Proposed Algorithm: Maximum Sample Reuse (MSR) Monte Carlo.} 
The simple MC method is sub-optimal since, for each sampled $S \in \mS_i$, the value of $U(S)$ and $U(S \cup i)$ are only used for estimating $\phi_\banz(i)$, i.e., the Banzhaf value of a single datum $i$. 
This inevitably results in a factor of $n$ in the final sample complexity as we need the same amount of samples to estimate each $i \in N$. 
To address this weakness, we propose an advanced MC estimator which achieves \emph{maximum sample reuse} (MSR). 
Specifically, by the linearity of expectation, we have
$
\phi_\banz(i) = \E_{S \sim \unif(2^{N \setminus i})} \left[  U(S \cup i) \right] - \E_{S \sim \unif(2^{N \setminus i})} \left[ U(S) \right] 
$.
Suppose we have $m$ samples $\mS = \{ S_1, \ldots, S_m \}$ i.i.d. drawn from $\unif(2^{N})$. For every data point $i$ we can divide $\mS$ into $\Sowni \cup \Snotowni$ where $\Sowni = \{S \in \mS: i \in S \}$ and $\Snotowni = \{S \in \mS: i \notin S\} = \mS \setminus \Sowni$. 
We can then estimate $\phi(i)$ by
\begin{align}
    \widehat \phi_\gt(i) = \frac{1}{|\Sowni|} \sum_{S \in \Sowni} U(S) - \frac{1}{|\Snotowni|} \sum_{S \in \Snotowni} U(S) \label{eq:group-testing}
\end{align}
or set $\widehat \phi_\gt(i)=0$ if either of $|\Sowni|$ and $|\Snotowni|$ is 0. 
In this way, the maximal sample reuse is achieved since \emph{all} evaluations of $U(S)$ are used in the estimation of $\phi(i)$ for every $i \in N$. 
We refer to this new estimator $\widehat \phi_\gt$ as \emph{MSR estimator}. Compared with Simple MC method, the MSR estimator saves a factor of $n$ in the sample complexity. 

\begin{theorem}
$\widehat \phi_\gt$ is an $(\eps, \delta)$-approximation to the exact Banzhaf value in $\ell_2$-norm with $O\left(\frac{n}{\eps^2}\log(\frac{n}{\delta})\right)$ calls of $U(\cdot)$, and in $\ell_\infty$-norm with $O\left(\frac{1}{\eps^2}\log(\frac{n}{\delta})\right)$ calls of $U(\cdot)$. 
\label{thm:gt}
\end{theorem}

\textbf{Proof overview.}
We remark that deriving this sample complexity is non-trivial. 
Unlike the simple Monte Carlo, the sizes of the samples that we average over in (\ref{eq:group-testing}) (i.e., $|\Sowni|$ and $|\Snotowni|$) are also random variables. Hence, we cannot simply apply Hoeffding's inequality to get a high-probability bound for $\|\widehat \phi_\gt - \phi_\banz\|$. The key to the proof is to notice that $|\Sowni|$ follows binomial distribution $\bin(m, 0.5)$.
Thus, we first show that $|\Sowni|$ is close to $m/2$ with high probability, and then apply Hoeffding's inequality to bound the difference between $\frac{1}{m/2} \sum_{S \in \Sowni} U(S) - \frac{1}{m/2} \sum_{S \in \Snotowni} U(S)$ and $\phi_\banz(i)$.

The actual estimator of the Banzhaf value that we build is based upon the noisy variant $\Uhat$. In Appendix \ref{appendix:msr-robustness}, we study the impact of noisy utility function evaluation on the sample complexity of the MSR estimator. It can be shown that our MSR algorithm has the same sample complexity with the noisy $\Uhat$, despite a small extra irreducible error.

\textbf{The existence of an efficient MSR estimator is unique for the Banzhaf value.}
Every semivalue can be written as the expectation of \emph{weighted} marginal contribution. 
Hence, one could construct an MSR estimator for arbitrary semivalue as follows:
$\widehat \phi_\gt(i) = \frac{2^{n-1}}{n|\Sowni|} \sum_{S \in \Sowni} w(|S|) U(S) - \frac{2^{n-1}}{n|\Snotowni|} \sum_{S \in \Snotowni}w(|S|+1) U(S)$. 
For the Shapley value, $w(|S|) = {n-1 \choose |S|-1}^{-1}$. This combinatorial coefficient makes the calculation of this estimator numerically unstable when $n$ is large. 
As we will show in the Appendix \ref{appendix:extend-msr}, it turns out that it is also impossible to construct a distribution $\D$ over $2^N$ s.t. $\phi_{\shap}\left(i\right) = \E_{S \sim \D|\D \own i} \left[ U(S) \right] - \E_{S \sim \D|\D \notown i} \left[ U(S) \right]$ for the Shapley value and \emph{any} other data value notions except the Banzhaf value. Hence, 
\emph{the existence of the efficient MSR estimator is a unique advantage of the Banzhaf value}.

\textbf{Lower Bound for Banzhaf Value Estimation. }
To understand the optimality of the MSR estimator, we derive a lower bound for any Banzhaf estimator that achieves $(\eps, \delta)$-approximation in $\ell_\infty$-norm. 
The main idea of deriving the lower bound is to use Yao's minimax principle. Specifically, we construct a distribution over instances of utility functions and prove that no deterministic algorithm can work well against that distribution.

\begin{theorem}
Every (randomized) Banzhaf value estimator that achieves $(\eps, \delta)$-approximation in $\ell_\infty$-norm for constant $\delta \in (0, 1/2)$ has sample complexity at least $\Omega(\frac{1}{\eps})$. 
\label{thm:lowerbound}
\end{theorem}


Recall that our MSR algorithm achieves $\tildeO(\frac{1}{\eps^2})$\footnote{Throughout the paper, we use $\tildeO$ to hide logarithmic factors.} sample complexity. This means that our MSR algorithm is close to optimal, with an extra factor of $\frac{1}{\eps}$.

\section{EVALUATION}
\label{sec:eval}

Our evaluation covers the following aspects: \textbf{(1)} Sample efficiency of the proposed MSR estimator for the Banzhaf value; \textbf{(2)} Robustness of the Banzhaf value compared to the six existing semivalue-based data value notions (including Shapley value, LOO error, and four representatives from Beta Shapley\footnote{We evaluate Beta(1, 4), Beta(1, 16), Beta(4, 1), Beta(16, 1) as the original paper.}); \textbf{(3)} Effectiveness of performing \emph{noisy label detection} and \emph{learning with weighted samples} based on the Banzhaf value. 
Detailed settings are provided in Appendix \ref{appendix:experiment}.

\begin{table*}[t]
\resizebox{\textwidth}{!}{
\begin{tabular}{@{}ccccccccc@{}}
\toprule
\textbf{Dataset} & \textbf{Data Banzhaf}  & \textbf{LOO}  & \textbf{Beta(16, 1)}   & \textbf{Beta(4, 1)}    & \textbf{Data Shapley}  & \textbf{Beta(1, 4)} & \textbf{Beta(1, 16)} & \textbf{Uniform} \\ \midrule
MNIST            & \textbf{0.745 (0.026)} & 0.708 (0.04)  & -                      & -                      & 0.74 (0.029)           & -                   & -                    & 0.733 (0.021)    \\
FMNIST           & \textbf{0.591 (0.014)} & 0.584 (0.02)  & -                      & -                      & 0.581 (0.017)          & -                   & -                    & 0.586 (0.013)    \\
CIFAR10       & \textbf{0.642 (0.002)} & 0.618 (0.005) & -                      & -                      & 0.635 (0.004)          & -                   & -                    & 0.609 (0.004)    \\
Click            & \textbf{0.6 (0.002)}   & 0.575 (0.005) & -                      & -                      & 0.589 (0.002)          & -                   & -                    & 0.57 (0.005)     \\
Fraud            & \textbf{0.923 (0.002)} & 0.907 (0.002) & 0.912 (0.004)          & 0.919 (0.005)          & 0.899 (0.002)          & 0.897 (0.001)       & 0.897 (0.001)        & 0.906 (0.002)    \\
Creditcard       & \textbf{0.66 (0.002)}  & 0.637 (0.006) & 0.646 (0.003)          & 0.658 (0.007)          & 0.654 (0.003)          & 0.643 (0.004)       & 0.629 (0.007)        & 0.632 (0.003)    \\
Vehicle          & \textbf{0.814 (0.003)} & 0.792 (0.008) & 0.796 (0.003)          & 0.806 (0.004)          & 0.808 (0.003)          & 0.805 (0.005)       & 0.8 (0.004)          & 0.791 (0.005)    \\
Apsfail          & 0.925 (0.0)            & 0.921 (0.003) & 0.924 (0.001)          & \textbf{0.926 (0.001)} & 0.921 (0.002)          & 0.92 (0.002)        & 0.919 (0.001)        & 0.921 (0.002)    \\
Phoneme          & \textbf{0.778 (0.001)} & 0.766 (0.006) & 0.765 (0.002)          & 0.766 (0.005)          & 0.77 (0.004)           & 0.767 (0.003)       & 0.766 (0.003)        & 0.758 (0.002)    \\
Wind             & \textbf{0.832 (0.003)} & 0.828 (0.002) & 0.827 (0.003)          & 0.831 (0.002)          & 0.825 (0.002)          & 0.823 (0.002)       & 0.823 (0.002)        & 0.825 (0.003)    \\
Pol              & \textbf{0.856 (0.005)} & 0.834 (0.008) & 0.837 (0.009)          & 0.848 (0.004)          & 0.836 (0.014)          & 0.824 (0.007)       & 0.812 (0.008)        & 0.841 (0.009)    \\
CPU              & 0.896 (0.001)          & 0.897 (0.002) & \textbf{0.899 (0.001)} & 0.897 (0.002)          & 0.894 (0.002)          & 0.892 (0.001)       & 0.889 (0.002)        & 0.895 (0.001)    \\
2DPlanes         & \textbf{0.846 (0.006)} & 0.83 (0.006)  & 0.837 (0.006)          & 0.841 (0.003)          & \textbf{0.846 (0.005)} & 0.843 (0.006)       & 0.838 (0.007)        & 0.829 (0.007)  \\ \bottomrule
\end{tabular}
}
\caption{
Accuracy comparison of models trained with weighted samples. We compare the seven data valuation methods on the 13 classification datasets. 
For MNIST, FMNIST, and CIFAR10 we use a size-2000 subset. 
The average and standard error of classification accuracy are denoted by 'average (standard error)'. The standard error is only due to the stochasticity in utility function evaluation. Boldface numbers denote the best method. 
Beta Shapley does NOT applicable for datasets with $\ge 1000$ data points (MNIST, FMNIST, CIFAR10, and Click) due to numerical instability. 
`Uniform' denotes training with uniformly weighted samples. 
}
\label{tb:weighted-sample}
\end{table*}

\subsection{Sample Efficiency}
\label{sec:eval-efficiency}

\begin{figure}[t]
    \centering
    \includegraphics[width=\columnwidth]{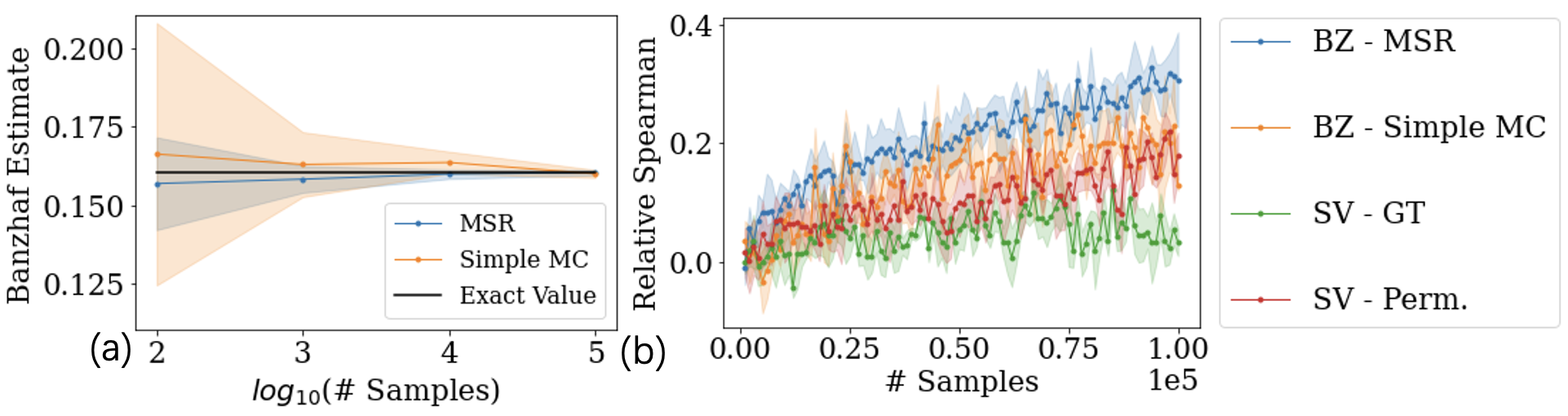}
    \caption{
    (a) Convergence of the MSR and simple MC estimator for the Banzhaf value of a data point from a synthetic dataset. The shaded area means the estimation variance over 5 runs of sampling. 
    (b) The Relative Spearman Index of the estimators for Shapley and Banzhaf value on the MNIST dataset. 
    `SV-GT' refers to Group Testing-based Shapley estimator, and `SV-Perm.' refers to Permutation Sampling Shapley estimator. 
    For `SV-GT', we implement the improved Group Testing-based Shapley estimator proposed in \cite{wang2023note} instead of the original version from \cite{jia2019towards}. 
    }
    \label{fig:convergence}
\end{figure}

\textbf{MSR vs. Simple MC.}
We compare the sample complexity of the MSR and the simple MC estimator for approximating the Banzhaf value. In order to exactly evaluate the estimation error of the two estimators, we use a synthetic dataset generated by multivariate Gaussian with only 10 data points---a scale where we can compute the Banzhaf value exactly. 
The utility function is the validation accuracy of logistic regression trained with full-batch gradient descent (no randomness in training). 
Thus, the randomness associated with the estimator error is solely from random sampling in the estimation algorithm. 
Figure \ref{fig:convergence} (a) compares the variance of the two estimators as the number of samples grows. 
As we can see, the estimation error of the MSR estimator reduces much more quickly than that of the simple MC estimator. Furthermore, given the same amount of samples, the MSR estimator exhibits a much smaller variance across different runs compared to the simple MC method.

\newcommand{\rsp}{\textbf{Relative Spearman Index}}

\textbf{Banzhaf vs. Shapley.} 
We then compare the two Banzhaf value estimators with two popular Shapley value estimators: the \emph{Permutation Sampling} \citep{castro2009polynomial} and \add{the \emph{Group Testing} algorithm \citep{jia2019towards, wang2023note}}. 
Since the Shapley and Banzhaf values are of different scales, for a fair comparison, we measure the consistency of the ranking of estimated data values. Specifically, we increase the sample size by adding a new batch of samples at every iteration, and evaluate each of the estimators on different sample sizes. For each estimator, we calculate the \rsp, which is the Spearman index of the value estimates between two adjacent iterations.
A high Relative Spearman Index means the ranking does not change too much with extra samples, which implies convergence of data value rankings. 
Figure \ref{fig:convergence} (b) compares the Relative Spearman Index of different data value estimators on MNIST dataset. 
We can see that the MSR estimator for the Banzhaf value converges much faster than the two estimators for the Shapley value in terms of data value rankings.

\begin{figure}[t]
    \centering
    \includegraphics[width=0.5\columnwidth]{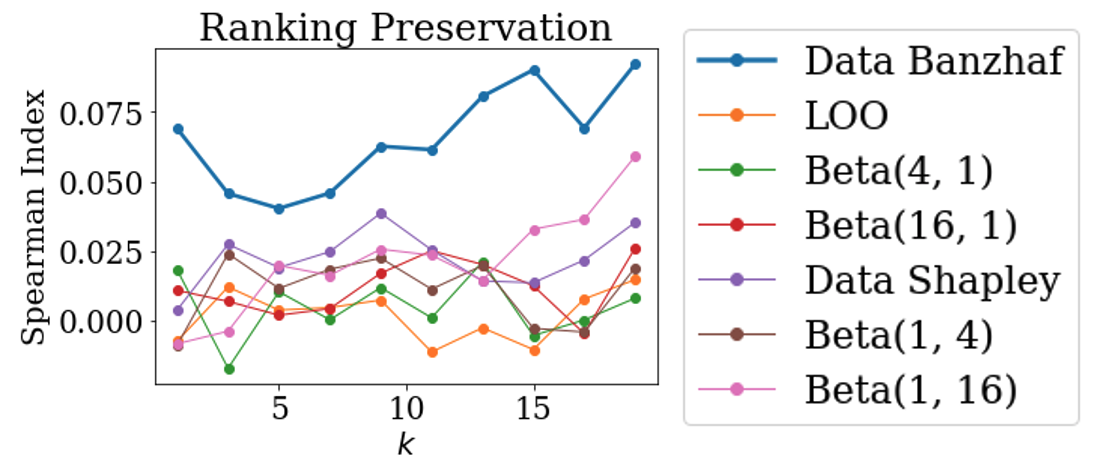}
    \caption{Impact of the noise in utility scores on the Spearman index between the ranking of reference data value and the ranking of data value estimated from noisy utility scores. The `$k$' in x-axis means the number of repeated evaluations of $U(S)$. \textbf{The larger the $k$, the smaller the noise magnitude.} Detailed experiment procedure are in Appendix \ref{appendix:additional-rank-stab}. }
    \label{fig:full}
\end{figure}

\begin{table*}[h]
\resizebox{\textwidth}{!}{
\begin{tabular}{@{}cccccccc@{}}
\toprule
{\color[HTML]{000000} \textbf{Dataset}} & {\color[HTML]{000000} \textbf{Data Banzhaf}}  & {\color[HTML]{000000} \textbf{LOO}}           & {\color[HTML]{000000} \textbf{Beta(16, 1)}} & {\color[HTML]{000000} \textbf{Beta(4, 1)}}    & {\color[HTML]{000000} \textbf{Data Shapley}}  & {\color[HTML]{000000} \textbf{Beta(1, 4)}} & {\color[HTML]{000000} \textbf{Beta(1, 16)}} \\ \midrule
{\color[HTML]{000000} MNIST}            & {\color[HTML]{000000} \textbf{0.193 (0.017)}} & {\color[HTML]{000000} 0.165 (0.009)}          & {\color[HTML]{000000} -}                    & {\color[HTML]{000000} -}                      & {\color[HTML]{000000} 0.135 (0.025)}          & {\color[HTML]{000000} -}                   & {\color[HTML]{000000} -}                    \\
{\color[HTML]{000000} FMNIST}           & {\color[HTML]{000000} 0.156 (0.018)}          & {\color[HTML]{000000} \textbf{0.164 (0.014)}} & {\color[HTML]{000000} -}                    & {\color[HTML]{000000} -}                      & {\color[HTML]{000000} 0.135 (0.016)}          & {\color[HTML]{000000} -}                   & {\color[HTML]{000000} -}                    \\
{\color[HTML]{000000} CIFAR10}          & {\color[HTML]{000000} \textbf{0.22 (0.003)}}  & {\color[HTML]{000000} 0.086 (0.02)}           & {\color[HTML]{000000} -}                    & {\color[HTML]{000000} -}                      & {\color[HTML]{000000} 0.152 (0.023)}          & {\color[HTML]{000000} -}                   & {\color[HTML]{000000} -}                    \\
{\color[HTML]{000000} Click}            & {\color[HTML]{000000} \textbf{0.206 (0.01)}}  & {\color[HTML]{000000} 0.096 (0.034)}          & {\color[HTML]{000000} -}                    & {\color[HTML]{000000} -}                      & {\color[HTML]{000000} 0.116 (0.024)}          & {\color[HTML]{000000} -}                   & {\color[HTML]{000000} -}                    \\
{\color[HTML]{000000} Fraud}            & {\color[HTML]{000000} 0.47 (0.024)}           & {\color[HTML]{000000} 0.157 (0.046)}          & {\color[HTML]{000000} 0.55 (0.032)}         & {\color[HTML]{000000} 0.59 (0.037)}           & {\color[HTML]{000000} \textbf{0.65 (0.032)}}  & {\color[HTML]{000000} 0.19 (0.058)}        & {\color[HTML]{000000} 0.14 (0.058)}         \\
{\color[HTML]{000000} Creditcard}       & {\color[HTML]{000000} 0.27 (0.024)}           & {\color[HTML]{000000} 0.113 (0.073)}          & {\color[HTML]{000000} 0.25 (0.063)}         & {\color[HTML]{000000} \textbf{0.28 (0.081)}}  & {\color[HTML]{000000} 0.26 (0.049)}           & {\color[HTML]{000000} 0.17 (0.024)}        & {\color[HTML]{000000} 0.17 (0.087)}         \\
{\color[HTML]{000000} Vehicle}          & {\color[HTML]{000000} \textbf{0.45 (0.0)}}    & {\color[HTML]{000000} 0.123 (0.068)}          & {\color[HTML]{000000} 0.43 (0.051)}         & {\color[HTML]{000000} 0.42 (0.068)}           & {\color[HTML]{000000} 0.41 (0.066)}           & {\color[HTML]{000000} 0.16 (0.058)}        & {\color[HTML]{000000} 0.1 (0.055)}          \\
{\color[HTML]{000000} Apsfail}          & {\color[HTML]{000000} \textbf{0.49 (0.037)}}  & {\color[HTML]{000000} 0.096 (0.09)}           & {\color[HTML]{000000} 0.36 (0.02)}          & {\color[HTML]{000000} 0.42 (0.024)}           & {\color[HTML]{000000} 0.47 (0.024)}           & {\color[HTML]{000000} 0.22 (0.051)}        & {\color[HTML]{000000} 0.2 (0.071)}          \\
{\color[HTML]{000000} Phoneme}          & {\color[HTML]{000000} 0.216 (0.023)}          & {\color[HTML]{000000} 0.115 (0.026)}          & {\color[HTML]{000000} 0.232 (0.02)}         & {\color[HTML]{000000} \textbf{0.236 (0.027)}} & {\color[HTML]{000000} 0.216 (0.032)}          & {\color[HTML]{000000} 0.124 (0.039)}       & {\color[HTML]{000000} 0.088 (0.02)}         \\
{\color[HTML]{000000} Wind}             & {\color[HTML]{000000} 0.36 (0.02)}            & {\color[HTML]{000000} 0.073 (0.022)}          & {\color[HTML]{000000} 0.51 (0.037)}         & {\color[HTML]{000000} 0.52 (0.04)}            & {\color[HTML]{000000} \textbf{0.57 (0.068)}}  & {\color[HTML]{000000} 0.19 (0.086)}        & {\color[HTML]{000000} 0.17 (0.06)}          \\
{\color[HTML]{000000} Pol}              & {\color[HTML]{000000} \textbf{0.47 (0.04)}}   & {\color[HTML]{000000} 0.097 (0.093)}          & {\color[HTML]{000000} 0.26 (0.037)}         & {\color[HTML]{000000} 0.4 (0.055)}            & {\color[HTML]{000000} 0.44 (0.058)}           & {\color[HTML]{000000} 0.17 (0.051)}        & {\color[HTML]{000000} 0.09 (0.02)}          \\
{\color[HTML]{000000} CPU}              & {\color[HTML]{000000} 0.35 (0.045)}           & {\color[HTML]{000000} 0.107 (0.074)}          & {\color[HTML]{000000} 0.45 (0.055)}         & {\color[HTML]{000000} \textbf{0.48 (0.06)}}   & {\color[HTML]{000000} 0.46 (0.037)}           & {\color[HTML]{000000} 0.13 (0.068)}        & {\color[HTML]{000000} 0.08 (0.081)}         \\
{\color[HTML]{000000} 2DPlanes}         & {\color[HTML]{000000} 0.422 (0.025)}          & {\color[HTML]{000000} 0.153 (0.057)}          & {\color[HTML]{000000} 0.338 (0.034)}        & {\color[HTML]{000000} 0.471 (0.041)}          & {\color[HTML]{000000} \textbf{0.512 (0.082)}} & {\color[HTML]{000000} 0.471 (0.041)}       & {\color[HTML]{000000} 0.338 (0.034)}        \\ \bottomrule
\end{tabular}
}
\caption{Comparison of mislabel data detection ability of the seven data valuation methods on the 13 classification datasets. 
The average and standard error of F1-score are denoted by `average (standard error)'. 
The standard error is only due to the random noise in the utility function evaluation. 
Boldface numbers denote the best method in F1-score average. 
}
\label{tb:mislabel-detection}
\end{table*}

\subsection{Ranking Stability under Noisy Utility Functions}
\label{sec:eval-ranking}

\add{We compare the robustness of data value notions in preserving the value ranking against the utility score perturbation due to the stochasticity in SGD}. 
\add{The tricky part in the experiment design is that \emph{we need to adjust the scale of the perturbation caused by natural stochastic learning algorithm}. 
In Appendix \ref{appendix:additional-rank-stab}, we show a procedure for controlling the magnitude of perturbation via repeatedly evaluating $U(S)$ for $k$ times: the larger the $k$, the smaller the noise magnitude in the averaged $\Uhat(S)$.  
In Figure \ref{fig:full}, we plot the Spearman index between the ranking of reference data values\footnote{We approximate the ground-truth by setting $k=50$.} and the ranking of data values estimated from noisy utility scores. Detailed settings are in Appendix \ref{appendix:additional-rank-stab}.}
As we can see, Data Banzhaf achieves the most stable ranking and its stability advantage gets more prominent as the noise increases. Moreover, we show the box-plot of Banzhaf value estimates in Figure \ref{fig:checkvar-value} (c). Compared with LOO and Shapley value, learning stochasticity has a much smaller impact on the ranking of Banzhaf values (the average Spearman index $\approx 0.856$ compared with the one for Shapley value $\approx 0.038$). 
Back to Figure \ref{fig:topk}, the high/low-quality data points selected by the Banzhaf value is much more consistent across different runs compared with LOO and Shapley value. 



\subsection{Applications of Data Banzhaf}
\label{sec:eval-applications}

Given the promising results obtained from the proof-of-concept evaluation in Section \ref{sec:eval-efficiency} and \ref{sec:eval-ranking}, we move forward to real-world applications and evaluate the effectiveness of Data Banzhaf in distinguishing data quality for machine learning tasks. Particularly, we considered two applications enabled by data valuation: one is to reweight training data during learning and another is to detect mislabeled points. We use neural networks trained with Adam as the learning algorithm wherein the associated utility function is noisy in nature. 
We remark that most prior works (e.g., \cite{ghorbani2019data} and \cite{kwon2021beta}) use deterministic Logistic Regression to avoid the randomness in data value results.
We compare with 6 baselines that are previously proposed semivalue-based data value notions: Data Shapley, Leave-one-out (LOO), and 4 variations of Beta Shapley \citep{kwon2021beta} (Beta(1, 4), Beta(1, 16), Beta(4, 1), Beta(16, 1)).\footnote{Beta Shapley does not apply for datasets with $\ge 1000$ data points due to numerical instability.} 
We use 13 standard datasets that are previously used in the data valuation literature as the benchmark tasks. 

\textbf{Learning with Weighted Samples.} 
Similar to \cite{kwon2021beta}, we weight each training point by normalizing the associated data value between [0,1]. Then, during training, each training sample will be selected with a probability equal to the assigned weight. As a result, data points with a higher value are more likely to be selected in the random mini-batch of SGD, and data points with a lower value are rarely used.
We train a neural network classifier to minimize the weighted loss, and then evaluate the accuracy on the held-out test dataset. 
As Table \ref{tb:weighted-sample} shows, Data Banzhaf outperforms other baselines.

\textbf{Noisy Label Detection.}
We investigate the ability of different data value notions in detecting mislabeled points under noisy utility functions. 
We generate noisy labeled samples by flipping labels for a randomly chosen 10\% of training data points. 
We mark a data point as a mislabeled one if its data value is less than 10 percentile of all data value scores. 
We use F1-score as the performance metric for mislabeling detection. 
Table \ref{tb:mislabel-detection} in the Appendix shows the F1-score of the 7 data valuation methods and Data Banzhaf shows the best overall performance. 
\section{Limitation and Future Work}
\label{sec:conclusion}

This work presents the first focused study on the reliability of data valuation in the presence of learning stochasticity. We develop Data Banzhaf, a new data valuation method that is more robust against the perturbations to the model performance scores than existing ones. 
One limitation of this study is that the robustness guarantee is a worst-case guarantee, which assumes the perturbation could be arbitrary or even adversarial. 
While such a robustness notion is advantageous in many aspects as discussed in Remark \ref{remark:why-noise-agonostic}, the threat model might be strong when the source of perturbation is known. Developing new robustness notions which take into account specific utility noise distributions induced by learning stochasticity is important future work, and it requires a breakthrough in understanding the noise structure caused by learning stochasticity.

\section*{Acknowledgments}
The work was supported by Princeton's Gordon Y. S. Wu Fellowship and Cisco Research Awards. 
We thank Sanjeev Arora for the helpful discussion. We thank Yongchan Kwon for sharing the implementation of Beta Shapley and dataset sources. We thank Rachel Li and Jacqueline Wei for sharing the summary of the discussion about this work in CS236r (Topics at the Interface between Computer Science and Economics, Fall 2022) at Harvard. We are grateful to anonymous reviewers at AISTATS for their valuable feedback. 

\newpage

\bibliographystyle{plainnat}
\bibliography{ref}


\newpage
\onecolumn

\appendix

\section{ Extended Related Work }
\label{appendix:related-work}

\paragraph{Cooperative Game Theory-based Data Valuation.}
Game-theoretic formulations of data valuation have become popular in recent years due to the formal, axiomatic justification. 
Particularly, the Shapley value has become a popular data value notion \citep{ghorbani2019data, jia2019towards} as it is the unique value notion that satisfies the four axioms: \emph{linearity, dummy player, symmetry, and efficiency}. Alternatives to the Shapley value for data valuation have also been proposed through the relaxation of the Shapley axioms. 
As we mentioned in the main text, by relaxing the efficiency axiom, the class of solution concepts that satisfy linearity, dummy player, and symmetry is called \emph{semivalue} \citep{weber1988probabilistic}. 
\cite{kwon2021beta} propose \emph{Beta Shapley}, which is a collection of semivalues that enjoy certain mathematical elegance in the representation. However, the construction of Beta Shapley does \emph{not} take the perturbation of performance scores into account; the original paper only uses deterministic learning algorithms (such as Logistic Regression) for the experiment. 
In this work, we characterize the Banzhaf value as a robust semivalue against the perturbation in the performance scores, which is critical for applications involving stochastic training such as neural network training.

In addition, by relaxing the linearity axiom of the Shapley value, \cite{yan2020ifyoulike} propose to use the \emph{Least core} \citep{deng1994complexity}, another classic concept in cooperative game theory, as an alternative to the Shapley value for data valuation. At a high level, the Least Core is a profit allocation scheme that requires the smallest subsidy to each coalition $S$ so that no participant has the incentive to deviate from the grand coalition $N$. It is computed by solving the linear programming problem below:
\begin{align}
\label{eq:leastcore} 
\min_{\phi_\lc} e  \quad \text{s.t. }\sum_{i=1}^n \phi_\lc(i) = U(N),\quad \sum_{i \in S}\phi_\lc(i) +e \geq U(S), \forall S \subseteq N  
\end{align}
We show additional experiment results for the robustness of the Least core in Appendix \ref{appendix:additional-rank-stab}. 

Distributional Shapley value \citep{ghorbani2020distributional,kwon2021efficient} is a variant of Data Shapley which measures the contribution of a data point with respect to a \emph{data distribution} instead of a static dataset. The stability notion discussed in the paper is in terms of the perturbation to the data point instead of model performance scores. \cite{bian2021energy} take a probabilistic treatment of cooperative games. Through mean-field variational inference in the energy-based model, they develop \emph{multiple step variational value} as a data value notion that satisfies null player, marginalism, and symmetry. The marginalism axiom requires a player’s payoffs to depend only on his own marginal contributions – whenever they remain unchanged, his payoffs should be unaffected. 
\cite{yona2021s} relax the assumption that the learning algorithm is fixed in advance in the previous work, and extend Shapley value to jointly quantify the contribution of data points and learning algorithms. It improves the stability of data value under domain shifts by attributing the responsibility to the learning algorithm. \cite{agussurja2022convergence} derive the convergence property of the Shapley value in parametric Bayesian learning games, and apply the result to establish an online collaborative learning framework that is asymptotically Shapley-fair. 


\paragraph{Banzhaf value, Banzhaf power index, and friends.}
What is known today as the Banzhaf value or Banzhaf power index was originally introduced by Lionel Penrose in 1946 \citep{penrose1946elementary}. It was reinvented by John F. Banzhaf III in 1964 \citep{banzhaf1964weighted}, and was reinvented once more by James Samuel Coleman in 1971 \citep{coleman1971control} before it became part of the mainstream literature. 
In the field of machine learning, the Banzhaf value has been previously applied to the problem of measuring feature importance \citep{datta2015influence, kulynych2017feature, sliwinski2019axiomatic, patel2021high, karczmarz2021improved}. While these works suggest that the Banzhaf value could be an alternative to the popular Shapley value-based model interpretation methods \citep{lundberg2017unified}, it remains unclear in which settings the Banzhaf value may be preferable to the Shapley value. This work provides the first theoretical understanding of the advantage of the Banzhaf value in terms of robustness. 
In addition, the empirical study by \citet{karczmarz2021improved} observes that the Banzhaf value is much more robust than the Shapley value when the numerical precision is low in the computation, which validates our theoretical result. 

We would like to note that the Banzhaf value is a generalization of the Banzhaf power index \citep{banzhaf1964weighted} which is designed for gauging the voting power of players in a simple voting game. In a \emph{simple voting game}, the utility function $U: 2^N \rightarrow \{0, 1\}$ where $U(\emptyset)=0$, $U(N)=1$ and $U(S) \le U(T)$ whenever $S \subseteq T$. 
In contrast, the setting of data valuation is more complicated and challenging as we do \emph{not} assume any particular structure of the utility function $U$. 
There are also many kinds of power indices available, such as Shapley-Shubik index \citep{shapley1953value}, Holler index \citep{holler1982forming}, and Deegan-Packel index \citep{deegan1978new}. 
The interpretation and computation of these power indices are active topics in cooperative game theory (e.g., \citep{holler1983power, aziz2008complexity}). 
In this work, we explore the most robust data value notion among the space of semivalues. Exploring the possibility of extending other kinds of cooperative solution concepts to data valuation is an interesting and promising future research direction. 

We also point out that the Banzhaf value is equivalent to Datamodel \citep{ilyas2022datamodels, saunshi2022understanding} if each data point is sampled independently with probability 0.5 and no regularization are used (i.e., $p=0.5$ in the version of Datamodel discussed in \cite{saunshi2022understanding}). This is due to a characterization of the Banzhaf value as the best linear approximation for the utility function in terms of least square loss \citep{hammer1992approximations}. 

\paragraph{Efficient Estimation of the Banzhaf and Shapley value.}
Most of the estimation algorithms for the Banzhaf and Shapley value are based on Monte Carlo techniques, especially when no prior knowledge about the structure of utility function $U$ is available. 
The Simple Monte Carlo estimation for Shapley value (i.e., the Permutation Sampling) was mentioned in very early works \citep{mann1960values}, and the sample complexity analysis of Permutation sampling for Shapley value can be found in \citep{maleki2015addressing}. 
\citet{covertshapley} propose an improved Shapley estimator based on the Importance Sampling technique. 
\cite{jia2019towards} improve the sample complexity of Monte Carlo-based Shapley estimation based on group testing technique (which is further improved by \cite{wang2023note} later). G-Shapley, TMC-Shapley \citep{ghorbani2019data} and KNN-Shapley \citep{jia2019efficient} have been proposed as the efficient proxies of Shapley value. However, these are \emph{biased} estimators for the Shapley value in nature. 
The sample complexity of the Simple Monte Carlo method for the Banzhaf value / Banzhaf power index \citep{merrill1982approximations} first appeared in \citet{bachrach2010approximating}. 

Another line of works studies the estimation of the Shapley and Banzhaf value in the problems with specific structures, e.g., (weighted) voting games \citep{owen1972multilinear, fatima2008linear, teneggi2021fast}, the games where only a few players have non-zero contribution \citep{jia2019towards, lin2022measuring}.

\paragraph{Alternative Approaches for Data Valuation in ML.} 
We review some recent works on data valuation methods here that are \emph{not} based on cooperative game theory, and we refer the readers to \cite{sim2022data} for a comprehensive technical survey of data valuation in ML. \cite{sim2020collaborative} use the reduction in uncertainty of the model parameters given the data as the valuation metric. Axioms that are important for collaborative learning such as strict desirability and monotonicity are also mentioned in the paper. 
\emph{Training-free} and \emph{task-agnostic} data valuation methods have also been proposed. 
\cite{tay2022incentivizing} suggest a data valuation method utilizing maximum mean discrepancy (MMD) between the data source and the true data distribution. \cite{xu2021validation} come up with a diversity measure called robust volume (RV) for valuing data sources. The robustness of the proposed data value notion is discussed in terms of the stability against data replication (via direct data copying). \cite{han2020replication} also study the replication robustness of semivalues. 
\cite{wu2022davinz} use a domain-aware generalization bound for data valuation, where the bound is based on neural tangent kernel (NTK) theory. 
\cite{amiri2022fundamentals} use the statistical differences between the source data and a baseline dataset as the valuation metric.

\newpage

\section{ Further Discussion about Considerations in Definition \ref{def:safety-margin} }
\label{appendix:robustness-discussion}



\subsection{Why do we consider rank stability?}
\label{appendix:disc-rank-stability}

As mentioned in the main text, rank stability is a reasonable robustness measure as the ranking of data values is important in many applications such as data subset selection and data pruning. 
Yet, another natural robustness measure is the stability in absolute value. Specifically, we can view a semivalue as a function $\phi: \R^{2^n} \rightarrow \R^n$ which takes a utility function $U \in \R^{2^n}$ as input, and output the values of data points $\phi(U) \in \R^n$. By taking this functional view, a natural robustness measure for semivalue $\phi(\cdot; w)$ is its Lipschitz constant $L$, which is defined as the smallest constant such that 
\begin{align}
    \norm{\phi(U; w) - \phi(\widehat U; w)} 
    \le 
    L \norm{U - \widehat U}
\end{align}
for all possible pairs of $U$ and $\widehat U$. 
However, since the efficiency axiom is relaxed for semivalue, such a robustness measure has the issue that \emph{different semivalues have different scales; the same change in the absolute value may mean differently for them.} 
On the other hand, rank stability provides a fair measure for comparison between different semivalues.

\subsection{Why do we consider a noise-structure-agnostic definition?}
\label{appendix:disc-noise-agnostic}

The safety margin defined in Definition \ref{def:safety-margin} does \textbf{not} depend on the actual noise induced by a specific stochastic training algorithm. 
Indeed, one may attempt to \textbf{directly analyze the potential perturbation distribution of utility function $U(S)$ caused by the stochasticity in learning algorithm} (e.g., random initialization, mini-batch selection), and define the safety margin \textbf{with respect to the specific perturbation (distribution)}. 

However, we did not adopt such a definition (referred to as \emph{noise-structure-specific} notion later) due to several considerations. 


\textbf{I. Even if the datasets and learning algorithm are known in advance, it is usually intractable to analytically derive a definite assertion on the noise in performance scores.} 
The probability distribution of the model performance scores $U(S) = \metric(\A(S))$ change with different training data $S$, test data (characterized by $\metric(\cdot)$), and the hyperparameters of learning algorithm $\A$. Even if such information is all available in advance (before we pick the data value notion), it is difficult to analyze the distribution of $U(S)$ without actually training on $S$ for common learning algorithms such as SGD. In order to illustrate the technical difficulty, we show such an attempt in Section \ref{sec:dist-in-simple-setting}. 

Specifically, under the simple (if not the simplest) setting of 1-dimensional linear regression trained by full batch gradient descent with Gaussian random initialization, we show the derivation of the probability distribution of validation mean squared error.\footnote{We use full batch gradient descent so the only randomness in the learning algorithm is the random initialization.} We show that the distribution of validation mean squared error is a \emph{generalized $\chi^2$ distribution}. Unfortunately, the probability density and cumulative density of generalized $\chi^2$ distribution are known for being intractable, which impedes further analysis of its impact on data values. Furthermore, \emph{it is unclear how to analyze the validation loss distribution for batch stochastic gradient descent (see the discussion at the end of Appendix \ref{sec:dist-in-simple-setting})}. 
As we can see, even for such a simple setting, there are significant challenges in designing noise-structure-specific robust data value notions.

\textbf{II. Noise-structure-specific robust data value notion may be computationally infeasible. }
More importantly, in order to design noise-structure-specific robust data value notion for a dataset $N$ of size $n$, we need to understand the noise distribution of $U(S)$ for every subset $S \subseteq N$, which introduces an exponentially large computationally burden. While one might be able to resort to numerical methods to approximate the intractable probability density issue mentioned before, the computational costs incurred by modeling the performance distribution on exponential subsets are prohibitive.

\textbf{III. In practice, one may not have prior knowledge about the dataset and source of perturbation; in that case, a worst-case robustness notion is preferred.}
Another difficulty regarding the noise-structure-specific robustness notion is that in practice, the datasets may not be known in advance for privacy consideration \citep{agahari2022not, tian2022private}. 
Furthermore, the performance scores may also be perturbed due to other factors such as hardware faults, software bugs, or even adversarial attacks. 
Given the uncertainty of the perturbation in the practical scenario, it makes more sense to define the robustness in terms of the worst-case perturbation. 
This is exactly Kerckhoffs's principle and such definitions are common in machine learning. For instance, robust learning against adversarial examples is aimed at being resilient to the worst-case noise \citep{madry2017towards} within a norm ball; differential privacy \citep{dwork2006calibrating} protects individual data record's information from arbitrary attackers.


Based on the identified difficulties for the potential noise-structure-specific notion, and the advantages of the noise-structure-agnostic notion, we believe it is ideal to define the robustness through the way in Section \ref{sec:safety-margin-def}. 
Importantly, the proposed robustness notion can not only lead to \textbf{tractable robustness analysis} for celebrated data value notions (Shapley, LOO); at the same time, the data value notion obtained by optimizing the proposed robustness measure, i.e., the Banzhaf value, indeed \textbf{achieves good robustness against realistic learning stochasticity} (as shown in the experiment section). Last but not least, the study of how to characterize the dependency between performance score and learning stochasticity itself needs to be investigated in depth before one could claim a reasonable probabilistic robustness definition under learning stochasticity.

\subsubsection{Difficulties in Analytical Analysis of $U(S)$}
\label{sec:dist-in-simple-setting}

\newcommand{\mL}{\mathcal{L}}
\newcommand{\model}{f}
\newcommand{\val}{\mathrm{val}}
\newcommand{\mP}{\mathcal{P}}

As we mentioned in Section \ref{sec:utility_stochastic}, the utility of a dataset $U(S)$ is defined as $\metric(\A(S))$, i.e., the performance of a model $\A(S)$ trained on a dataset $S$. 
However, the learning algorithms may have randomness during the training process. 
For instance, the stochastic gradient descent (SGD) algorithm involves random weights initialization and random mini-batch selection. 
Formally, we can view $\A$ as a randomized function that takes a dataset $S$ as input, and output a trained model $\A(S; r)$, where $r$ is a random string that describes all randomness used during the training process. 
For SGD, $r$ is the weights initialization and mini-batch selection choices. 
For each execution of the learning algorithm, $r$ is sampled from the corresponding probability distribution $\mP_\A$ specified by the learning algorithm, e.g., the weights are initialized by isotropic Gaussian distribution and mini-batch selection is based on Binomial distribution. 
Since the trained model $\A(S)$ is a random variable, the utility $U(S)$ is inherently a randomized function and the randomness depends on $r$. 
To make data value notions such as Shapley and Banzhaf value to be well-defined and independent from the learning stochasticity, at the end of Section \ref{sec:utility_stochastic} we refine $U(S) := \E_{r \sim P}[ \metric(\A(S; r)) ]$. 
Let $\widehat U(S)$ denote the random utility $\metric(\A(S; r))$ with a randomly $r \sim \mP$. 
In order to find the most robust semivalue against the random noise $\widehat U(S) - U(S)$, one natural idea would be analytically derive the probability distribution of $\widehat U(S)$, and design the corresponding data value accordingly. 

Unfortunately, it is actually non-trivial to conclude a definite assumption on the performance noise caused by learning stochasticity. 
In this section, we illustrate such difficulty by directly analyzing the distribution of $\widehat U(S)$ for arguably the simplest setting: $1$-dimensional linear regression trained by gradient descent with random initialization. 
The model here is defined as $\model(x; \theta) := \theta x$ and is trained on a dataset $S = \{(x, y)\}$ via mean squared error $\mL(\theta; S) = \frac{1}{2}\sum_{(x, y) \in S}(y-\model(x; \theta))^2$. 
The space of both input feature $x$ and model parameter $\theta$ are $\R$. 

The source of the randomness in the learning algorithm here is random initialization.
Specifically, we use gradient descent to train the linear regression model with Gaussian random initialization. 
\begin{align*}
    \theta_0 &\sim \N(0, \sigma^2) \\
    \theta_1 &\leftarrow \theta_0 - \eta \g \mL(\theta_0) \\
    &\ldots \\
    \theta_T &\leftarrow \theta_0 - \eta \g \mL(\theta_{T-1}) 
\end{align*}
where $\eta$ is the learning rate and $T$ is the total number of iterations. 
The utility of the trained model is given by the mean squared error $\widehat U(S) := \mL(\theta_T; S_\val)$ on validation set $S_{\val} = \{(x^*, y^*)\}$. 

Due to the simplicity of the setting, we can derive the evolution of the distribution of $\theta_T$ as well as $\mL(\theta_T; S_\val)$. 
Since $\g \mL(\theta) = \sum_{(x, y) \in S} x (x \theta - y)$, at iteration $t$ we have
\begin{align*}
    \theta_{t+1} 
    &= \theta_{t} - \eta \sum_{(x, y) \in S} x (x \theta_{t} - y) \\
    &= \theta_t \left(1 - \eta \sum_{(x, y) \in S} x^2\right) + \eta \sum_{(x, y) \in S} xy 
\end{align*}
Therefore, $\{ \theta_t \}$ is a sequence of Gaussian random variables, where 
\begin{align*}
\E[\theta_{t+1}] &= \E[\theta_{t}] \left(1 - \eta \sum_{(x, y) \in S} x^2\right) + \eta \sum_{(x, y) \in S} xy  \\
\Var( \theta_{t+1} ) &= \Var( \theta_{t} )  \left(1 - \eta \sum_{(x, y) \in S} x^2\right)^2
\end{align*}

To simplify the notation, let $f(S) = \sum_{(x, y) \in S} xy$ and $g(S) = \sum_{(x, y) \in S} x^2$. 
By a simple analysis, we can derive the general term formula for $\theta_T$ as 
\begin{align*}
    \E[\theta_{T}] &= \frac{f(S)}{g(S)} \left( 1 - (1-\eta g(S))^T \right) \\
    \Var( \theta_{T} ) &= \sigma^2 \left(1 - \eta g(S) \right)^{2T}
\end{align*}

For each validation data point $(x^*, y^*)$, we have 
\begin{align*}
    x^* \theta_T - y^* \sim 
    \N \left( x^*\E[\theta_T] - y^*, \Var( \theta_{T} ) (x^*)^2 \right)
\end{align*}

Therefore, $(x^* \theta_T - y^*)^2$ is a (scaled) non-central $\chi^{\prime 2}$ distribution with degree of freedom 1 and non-centrality parameter $\frac{ \left( x^*\E[\theta_T] - y^*\right)^2 }{\Var(\theta_T) (x^*)^2}$. 
(\cite{abramowitz1964handbook}, Section 26.4.25):
\begin{align*}
    (x^* \theta_T - y^*)^2 &\sim \Var(\theta_T) (x^*)^2 \cdot \chi^{\prime 2}\left(1, \left( \frac{x^*\E[\theta_T] - y^*}{\std(\theta_T) (x^*)} \right)^2 \right) \\
    &\sim \Var(\theta_T) (x^*)^2 \cdot \chi^{\prime 2}\left(1, \frac{ \left( x^*\E[\theta_T] - y^*\right)^2 }{\Var(\theta_T) (x^*)^2} \right)
\end{align*}

Consequently, the validation loss $\mL(\theta_T; S_\val) = \sum_{(x^*, y^*) \in S_{\val}} (x^* \theta_T - y^*)^2$ is a generalized chi-squared distribution \citep{davies1980algorithm}:
\begin{align*}
    \mL(\theta_T; S_\val) 
    \sim \sum_{(x^*, y^*) \in S_\val} \Var(\theta_T) (x^*)^2 \cdot \chi^{\prime 2}\left(1, \frac{ \left( x^*\E[\theta_T] - y^*\right)^2 }{\Var(\theta_T) (x^*)^2} \right)
\end{align*}

There are two major difficulties in applying the above results for designing robust data value notions specific to such a simplified setting:
\begin{enumerate}
    \item The generalized chi-squared distribution is known for intractable probability density or cumulative distribution \citep{davies1980algorithm, das2021method}. 
    \item To design noise-structure-specific robust data value notion for a dataset $N$ of size $n$, we need to understand the noise distribution for every subset $S \subseteq N$, which introduces an exponentially large computational burden. While one might be able to resort to numerical methods to approximate the intractable probability density, the computational costs incurred by modeling the performance distribution on exponential subsets are prohibitive. 
    
\end{enumerate}



Therefore, even for such a very simple setting, there are significant difficulties in designing noise-structure-specific robust data value notions based on directly analyzing noise distribution. 
\paragraph{Difficulties in Extending to Batch Stochastic Gradient Descent.}
Furthermore, it is unclear how to extend the above analysis to batch stochastic gradient descent. 
For the case of batch stochastic gradient descent, while it is easy to see that the parameter in the first iteration $\theta_1$ is a Gaussian mixture, the distribution of parameters after the first iteration is intractable to analyze. 

Based on the above-identified difficulties, it is preferable to define the robustness in a noise-structure-agnostic way as we did in Section \ref{sec:databanzhaf}. 


\subsection{Why do we consider a $U$-structure-agnostic definition?}
\label{appendix:disc-U-agnostic}
This consideration shares the same reason as \textbf{II} and \textbf{III} in Appendix \ref{appendix:disc-noise-agnostic}. 
If the safety margin depends on the specific $U$, it means that we need to know about $U(S)$ for every $S \subseteq N$, which is computationally infeasible or even impossible. 
On the other hand, the definition of ``$\tau$-distinguishable'' utility functions characterizes the collection of $U$s such that the semivalue should be robust on, while also leading to tractable robustness analysis for semivalues.

\newpage

\section{ Proofs and Additional Theoretical Results }
\label{appendix:proof}

We provide a summary of the content in this section for the convenience of the readers. 
\begin{itemize}
    \item Appendix \ref{appendix:proof-in-main-text}: Proofs for the theorems appeared in the maintext. 
    \item Appendix \ref{appendix:extend-msr}: The MSR Estimator does not Extend to the Shapley Value and Other Known Semivalues. 
    \item Appendix \ref{appendix:msr-robustness}: Robustness of the MSR Estimator. 
    \item Appendix \ref{sec:l2-robustness}: Stability of Banzhaf value in $\ell_2$-norm. 
\end{itemize}

\subsection{Proofs for Theorem \ref{thm:safe-loo-shap}, \ref{thm:banzhaf-robustness}, \ref{thm:mc}, \ref{thm:gt}, \ref{thm:lowerbound} in the Main Text}
\label{appendix:proof-in-main-text}






We omit the parameters of $U, N$, or $w$ when it's clear from the context.

\subsubsection{The Safety Margin for the LOO error, the Shapley value, and the Banzhaf value}

\begin{lemma}
Given a semivalue with weight function $w(\cdot)$, we have 
\begin{align}
    \safe(\tau; w) = \tau
    \sqrt{
    \frac{ 
        \left(
        \sum_{k=1}^{n-1} {n-2 \choose k-1} \left(w(k)+w(k+1)\right)
        \right)^2
        }{ 
        \sum_{k=1}^{n-1} {n-2 \choose k-1} \left(w(k)+w(k+1)\right)^2
        }
    }
    \label{eq:safetymargin}
\end{align} 
for any $\tau>0$.
\label{lemma:safe}
\end{lemma}

\begin{proof}
For any $\tau>0$ and any pair of $(i, j)$, we denote 
\begin{align}
\safe_{i, j}(\tau; w) = \min_{U \in \Utau} \min_{\Uhat \in \{\Uhat: D_{i, j}(U; w) D_{i, j}(\Uhat; w) \le 0\}} \norm{\Uhat - U} \nonumber
\end{align}
as the minimum amount of noise that is required to reverse the ranking of $(i, j)$ among all utility functions that $\tau$-distinguish $(i, j)$. 
Thus, the safety margin of the semivalue $w$ is 
\begin{align}
    \safe(\tau; w) = \min_{i \ne j} \safe_{i, j}(\tau; w) \nonumber
\end{align}

Note that $D_{i, j}(U; w)$ can be written as a dot product of $U$ and a column vector $a \in \R^{2^n}$
\begin{align}
    D_{i, j}(U; w) = a^T U \nonumber
\end{align}
where each entry of $a$ corresponds to a subset $S \subseteq N$. We use $a[\cdot]$ to denote the value of $a$'s entry corresponds to $S$.  
For all $S \subseteq N\setminus \{i, j\}$, $a[S \cup i] = w(|S|+1) + w(|S|+2)$ and $a[S \cup j] = -(w(|S|+1) + w(|S|+2))$, and for all other subsets $a[S]=0$. 
Let the perturbation $x = \Uhat - U$ and matrix $A = a a^T$. 
\begin{align*}
    D_{i, j}(U; w) D_{i, j}(\Uhat; w) 
    &= (a^T U) (a^T \Uhat) \\
    &= (a^T U)^T (a^T \Uhat) \\ 
    &= U^T a a^T \Uhat \\
    &= U^T A \Uhat \\
    &= U^T A (U + x)
\end{align*}

Thus, if $D_{i, j}(U; w) D_{i, j}(\Uhat; w) \le 0$, the size of the perturbation $x$ must be at least
\begin{align}
    \norm{x} 
    &\ge \frac{|U^T A U|}{\norm{U^T A}} \nonumber \\
    &= \frac{|U^T A U|}{\sqrt{U^T A A U}} \nonumber \\
    &= \frac{|U^T A U|}{ \sqrt{a^T a} \sqrt{|U^T A U|} } \label{eq:temp} \\
    &= \sqrt{ \frac{|U^T A U|}{a^T a} } \nonumber
\end{align}
where (\ref{eq:temp}) is because 
$AA = (a a^T)(a a^T) = a (a^Ta) a^T = (a^Ta) a a^T = (a^Ta) A$. 
This lower bound is achievable when we set $x$ on the direction of $U^T A$. 
Therefore, we have 
\begin{align*}
\safe_{i, j}(\tau; w) = \min_{U \in \Utau} \sqrt{ \frac{|U^T A U|}{a^T a} }
\end{align*}

To make the notations less cumbersome, denote $f(S) = w(|S|+1) + w(|S|+2)$, and $g(S) = U(S\cup i) - U(S \cup j)$. By expanding the expression, we have 
\begin{align}
    \frac{|U^T A U|}{a^T a}
    &= \frac{\left|
    \sum_{S_1 \subseteq N \setminus \{i, j\} }
    \sum_{S_2 \subseteq N \setminus \{i, j\} }
    f(S_1)f(S_2)g(S_1)g(S_2)\right|
    }{ \sum_{S \subseteq N \setminus \{i, j\} } f^2(S) }\nonumber \\
    &= \frac{ \left| \left(
    \sum_{S_1 \subseteq N \setminus \{i, j\} }
    f(S_1)g(S_1) \right)
    \left(
    \sum_{S_2 \subseteq N \setminus \{i, j\} }
    f(S_2)g(S_2)\right) \right|
    }{ \sum_{S \subseteq N \setminus \{i, j\} } f^2(S) }\nonumber \\
    &= \frac{
    \left(
    \sum_{S \subseteq N \setminus \{i, j\} } f(S)g(S)
    \right)^2
    }{\sum_{S \subseteq N \setminus \{i, j\} } f^2(S)} \nonumber \\
    &= \frac{ 
    \left(
    \sum_{k=1}^{n-1} \left(w(k)+w(k+1)\right) 
    \sum_{S \subseteq N \setminus \{i, j\}, |S|=k-1}
    \left( U(S\cup i) - U(S\cup j) \right)
    \right)^2
    }{ 
    \sum_{k=1}^{n-1} {n-2 \choose k-1} \left(w(k)+w(k+1)\right)^2
    } \nonumber \\
    &= \frac{ 
    \left(
    \sum_{k=1}^{n-1} {n-2 \choose k-1} \left(w(k)+w(k+1)\right) {n-2 \choose k-1}^{-1}
    \sum_{S \subseteq N \setminus \{i, j\}, |S|=k-1}
    \left( U(S\cup i) - U(S\cup j) \right)
    \right)^2
    }{ 
    \sum_{k=1}^{n-1} {n-2 \choose k-1} \left(w(k)+w(k+1)\right)^2
    } \label{eq:lowerbound}
\end{align}

Clearly, the minimum of (\ref{eq:lowerbound}) is achieved when 
\begin{align*}
    {n-2 \choose k-1}^{-1}
    \sum_{S \subseteq N \setminus \{i, j\}, |S|=k-1}
    \left( U(S\cup i) - U(S\cup j) \right) = \tau
\end{align*}
for all $k$, i.e., 
$\safe_{i, j}(\tau; w) = 
\tau
\sqrt{
\frac{ 
    \left(
    \sum_{k=1}^{n-1} {n-2 \choose k-1} \left(w(k)+w(k+1)\right)
    \right)^2
    }{ 
    \sum_{k=1}^{n-1} {n-2 \choose k-1} \left(w(k)+w(k+1)\right)^2
    }
}$. 
This holds for every pair of data points $(i, j)$, which leads to our conclusion where $\safe(\tau; w) = 
\tau
\sqrt{
\frac{ 
    \left(
    \sum_{k=1}^{n-1} {n-2 \choose k-1} \left(w(k)+w(k+1)\right)
    \right)^2
    }{ 
    \sum_{k=1}^{n-1} {n-2 \choose k-1} \left(w(k)+w(k+1)\right)^2
    }
}$. 
\end{proof}

\paragraph{The Safety Margin for the LOO error and the Shapley value. }

\begin{customthm}{\ref{thm:safe-loo-shap}}[restated]
For any $\tau>0$, Leave-one-out error ($w_{\loo}(k) = n \iden[k=n]$) achieves $\safe(\tau; w_{\loo})=\tau$, and Shapley value ($w_{\shap}(k) = {n-1 \choose k-1}^{-1}$) achieves $\safe(\tau; w_{\shap}) = \tau \frac{n-1}{\sqrt{\sum_{k=1}^{n-1} {n-2 \choose k-1}^{-1} }}$. 
\end{customthm}
\begin{proof}
By plugging in $w_{\loo}(k) = n \iden[k=n]$ to (\ref{eq:safetymargin}), we have $\safe(\tau; w_{\loo})=\tau$. 
By plugging in $w_{\shap}(k) = {n-1 \choose k-1}^{-1}$ to (\ref{eq:safetymargin}), we have $\safe(\tau; w_{\shap}) = \tau \frac{n-1}{\sqrt{\sum_{k=1}^{n-1} {n-2 \choose k-1}^{-1} }}$. 
\end{proof}

\paragraph{Banzhaf Value Achieves the Largest Safety Margin.}

\begin{customthm}{\ref{thm:banzhaf-robustness}}[restated]
For any $\tau>0$, the Banzhaf value ($w(k)=\frac{n}{2^{n-1}}$) achieves the largest safety margin $\safe(\tau; w)=\tau 2^{n/2-1}$ among all semivalues. 
\end{customthm}
\begin{proof}
By Lemma \ref{lemma:safe}, we want to find the optimal semivalue weight function $w$ that maximizes 
$\frac{ 
    \left(
    \sum_{k=1}^{n-1} {n-2 \choose k-1} \left(w(k)+w(k+1)\right)
    \right)^2
    }{ 
    \sum_{k=1}^{n-1} {n-2 \choose k-1} \left(w(k)+w(k+1)\right)^2
}$. Notice that 
\begin{align}
    \frac{ 
    \left(
    \sum_{k=1}^{n-1} {n-2 \choose k-1} \left(w(k)+w(k+1)\right)
    \right)^2
    }{ 
    \sum_{k=1}^{n-1} {n-2 \choose k-1} \left(w(k)+w(k+1)\right)^2
    }
    &= \frac{ 
    \left(
    \sum_{k=1}^{n-1} \sqrt{{n-2 \choose k-1}}\sqrt{{n-2 \choose k-1}} \left(w(k)+w(k+1)\right)
    \right)^2
    }{ 
    \sum_{k=1}^{n-1} {n-2 \choose k-1} \left(w(k)+w(k+1)\right)^2
    } \nonumber \\
    &\le \frac{ 
    \sum_{k=1}^{n-1} {n-2 \choose k-1}
    \sum_{k=1}^{n-1} {n-2 \choose k-1} \left(w(k)+w(k+1)\right)^2
    }{ 
    \sum_{k=1}^{n-1} {n-2 \choose k-1} \left(w(k)+w(k+1)\right)^2
    } \label{eq:cauchy} \\
    &= \sum_{k=1}^{n-1} {n-2 \choose k-1} \nonumber \\
    &= 2^{n-2} \nonumber
\end{align}
where (\ref{eq:cauchy}) is due to Cauchy-Schwarz inequality. 

Note that this upper bound is achievable whenever $w(|S|+1) + w(|S|+2)$ is a constant due to the equality condition of Cauchy-Schwarz, which the weight function of Banzhaf value clearly satisfies. 
Therefore, the Banzhaf value achieves the largest $\safe(\tau; w)$ among all possible semivalues. 
\end{proof}

\textbf{Further Discussion of Theorem \ref{thm:banzhaf-robustness}.}
The safety margin is the largest noise in utility that can be tolerated such that the ranking of \emph{exact} semivalues calculated from clean utility match with that of \emph{exact} semivalues calculated from a noisy utility. However, calculating exact semivalues for a given utility function is NP-hard in general, and in practice, one often resorts to evaluating the utility function at limited sampled subsets and then using these limited samples to approximate semivalues. Hence, a natural question to ask is whether we can characterize the maximally-tolerable utility noise on the \emph{limited} sampled subsets such that the ranking of \emph{approximate} semivalues calculated from the clean utility samples align with that of \emph{approximate} semivalues calculated from the noisy samples. However, one issue with this type of characterization is that the ``safety margin'' in this case depends on \emph{both} the expression of the semivalue (i.e., $w$ that parameterizes the semivalue), as well as the underlying estimation algorithm for that semivalue. Since different semivalues have different estimation algorithms, such a result for different semivalues is not really comparable. 
On the other hand, our result in Theorem~\ref{thm:banzhaf-robustness} lifts the dependence on the underlying estimation algorithm. As a consequence, it allows one to compare the robustness between different semivalues.

We also note that, the Banzhaf value is \emph{not} the unique semivalue that achieves the maximal robustness in the setup of Theorem \ref{thm:banzhaf-robustness}. Any semivalues with a weight function $w$ s.t. $w(k)+w(k+1)$ is a constant also achieve the same safety margin. 
Such a semivalue must have $w(1) = w(3) = w(5) = \ldots$ and $w(2) = w(4) = w(6) = \ldots$. However, there's no natural explanation for why the semivalue should weigh odd and even cardinalities differently. Hence, the Banzhaf value is the only ``reasonable'' semivalue with maximal robustness. 

\subsubsection{Sample Complexity of Simple MC and MSR Estimator.}

\begin{customthm}{\ref{thm:mc}}[restated]
$\widehat \phi_\mc$ is an $(\eps, \delta)$-approximation to the exact Banzhaf value in $\ell_2$-norm with $O(\frac{n^2}{\eps^2}\log(\frac{n}{\delta}))$ calls of $U(\cdot)$, and in $\ell_\infty$-norm with $O(\frac{n}{\eps^2}\log(\frac{n}{\delta}))$ calls of $U(\cdot)$. 
\end{customthm}

\begin{proof}
Let $\mS = \{S_1, \ldots, S_m\}$ be the samples used for computing $\widehat \phi_\mc(i)$. Since the marginal contribution $U(S \cup i)-U(S)$ is always bounded between $[-1, 1]$, by Hoeffding, we have
\begin{align*}
    \Pr \left[ \left|\widehat \phi_\mc(i) - \phi(i)\right| \ge \eps \right] 
    \le 2 \exp \left( -2m \eps^2 \right)
\end{align*}
which holds for every $i \in N$. 

Thus, with union bound, for $\ell_2$-norm we have 
\begin{align*}
    \Pr_{\widehat \phi_\mc} \left[ \norm{\widehat \phi_\mc - \phi}_2 \ge \eps \right] &= 1 - \Pr_{\widehat \phi_\mc} \left[ \norm{\widehat \phi_\mc - \phi}_2^2 \le \eps^2 \right] \\
    &\le 1 - \Pr_{\widehat \phi_\mc} \left[ \cap_i \left|\widehat \phi_\mc(i) - \phi(i)\right| \le \eps / \sqrt{n} \right] \\
    &= \Pr_{\widehat \phi_\mc} \left[ \cup_i \left|\widehat \phi_\mc(i) - \phi(i)\right| \ge \eps / \sqrt{n} \right] \\
    &\le \sum_{i=1}^n \Pr_{\widehat \phi_\mc} \left[ \left|\widehat \phi_\mc(i) - \phi(i)\right| \ge \eps / \sqrt{n} \right] \\
    &\le 2 n \exp \left( -2m \eps^2 / n \right)
\end{align*}
By setting $2n \exp \left( -2m \eps^2 / n \right) \le \delta$, we get $m \ge \frac{n}{2\eps^2}\log\left(\frac{2n}{\delta}\right) = O\left(\frac{n}{\eps^2}\log(\frac{n}{\delta})\right)$. However, this $m$ only corresponds to the number of samples used to estimate a single $\phi(i)$, so the total number of samples required is $O\left(\frac{n^2}{\eps^2}\log(\frac{n}{\delta})\right)$. 

For $\ell_\infty$-norm we have 
\begin{align*}
    \Pr_{\widehat \phi_\mc} \left[ \norm{\widehat \phi_\mc - \phi}_\infty \ge \eps \right] 
    &= \Pr_{\widehat \phi_\mc} \left[ \cup_i \left|\widehat \phi_\mc(i) - \phi(i)\right| \ge \eps \right] \\
    &\le \sum_{i=1}^n \Pr_{\widehat \phi_\mc} \left[ \left|\widehat \phi_\mc(i) - \phi(i)\right| \ge \eps \right] \\
    &\le 2 n \exp \left( -2m \eps^2 \right)
\end{align*}
By setting $2n \exp \left( -2m \eps^2 \right) \le \delta$, we get $m \ge \frac{1}{2\eps^2}\log\left(\frac{2n}{\delta}\right) = O\left(\frac{1}{\eps^2}\log(\frac{n}{\delta})\right)$. However, this $m$ only corresponds to the number of samples used to estimate a single $\phi(i)$, so the total number of samples required is $O\left(\frac{n}{\eps^2}\log(\frac{n}{\delta})\right)$. 
\end{proof}




\begin{customthm}{\ref{thm:gt}}[restated]
$\widehat \phi_\gt$ is an $(\eps, \delta)$-approximation to the exact Banzhaf value in $\ell_2$-norm with $O\left(\frac{n}{\eps^2}\log(\frac{n}{\delta})\right)$ calls of $U(\cdot)$, and in $\ell_\infty$-norm with $O\left(\frac{1}{\eps^2}\log(\frac{n}{\delta})\right)$ calls of $U(\cdot)$. 
\end{customthm}
\begin{proof}
Since $\mS = \{ S_1, \ldots, S_m \}$ each i.i.d. drawn from $\unif(2^{N})$, it is easy to see that the size of sampled subsets that include data point $i$ follows binomial distribution $|\Sowni| \sim \bin(m, 0.5)$, and $|\Sowni| = m - |\Snotowni|$. 

We first define an alternative estimator 
\begin{align*}
    \widetilde{\phi}(i) = 
    \frac{1}{m/2} \sum_{S \in \Sowni} U(S) - \frac{1}{m/2} \sum_{S \in \Snotowni} U(S)
\end{align*}
which is independent of $|\Sowni|$ and $|\Snotowni|$. 
When both $|\Sowni|$ and $|\Snotowni| > 0$, we have
\begin{align}
    \left| \widehat \phi(i) - \widetilde{\phi}(i) \right| 
    &= \left| \left( \frac{1}{|\Sowni|} - \frac{1}{m/2} \right) \sum_{S \in \Sowni} U(S) - \left( \frac{1}{|\Snotowni|} - \frac{1}{m/2} \right) \sum_{S \in \Snotowni} U(S)
    \right| \nonumber \\
    &\le \left| 1 - \frac{2|\Sowni|}{m} \right| + \left| 1 - \frac{2|\Snotowni|}{m} \right| \label{eq:Ubound} \\
    &= 2 \left| 1 - \frac{2|\Sowni|}{m} \right| \label{eq:merge} \\
    &= \frac{4}{m} \left| |\Sowni| - \frac{m}{2} \right| \nonumber
\end{align}
where (\ref{eq:Ubound}) is due to $U(S) \le 1$ and (\ref{eq:merge}) is due to $|\Sowni| = m - |\Snotowni|$. 
When one of $|\Sowni|$ and $|\Snotowni|=0$, this upper bound also clearly holds. 

Since $|\Sowni| \sim \bin(m, 0.5)$, by Hoeffding inequality we have
\begin{align*}
    \Pr \left[ \left| |\Sowni| - \frac{m}{2} \right| \ge \Delta \right]
    \le 2 \exp \left(- \frac{2\Delta^2}{m}\right)
\end{align*}

Hence, with probability at least $1-2 \exp \left(- \frac{2\Delta^2}{m}\right)$, we have 
\begin{align*}
    \left| \widehat \phi(i) - \widetilde{\phi}(i) \right| 
    \le \frac{4\Delta}{m}
\end{align*}

\newcommand{\sign}{\mathrm{sign}}

Since 
\begin{align*}
    \widetilde{\phi}(i) 
    &= \frac{2}{m} \left( \sum_{S \in \Sowni} U(S) - \sum_{S \in \Snotowni} U(S) \right) \\
    &= \frac{1}{m} \sum_{S \in \mS} 2U(S) \sign(i, S)
\end{align*}
where $\sign(i, S) = 2\iden[i \in S]-1 \in \{\pm 1\}$. 
Thus, $2U(S) \sign(i, S) \in [-2, 2]$ and we can apply Hoeffding to bound the tail of $|\widetilde{\phi}(i)-\phi(i)|$:
\begin{align*}
    \Pr \left[ |\widetilde{\phi}(i)-\phi(i)| \ge t \right]
    \le 
    2 \exp \left( - \frac{mt^2}{8} \right)
\end{align*}

Now we bound $|\widehat \phi(i) - \phi(i)|$ as follows:
\begin{align*}
    &\Pr \left[ |\widehat \phi(i) - \phi(i)| \ge \eps \right] \\
    &= \Pr \left[ |\widehat \phi(i) - \phi(i)| \ge \eps | \left| |\Sowni| - \frac{m}{2} \right| \le \Delta \right] \Pr \left[ \left| |\Sowni| - \frac{m}{2} \right| \le \Delta \right] \\
    &~~~~+ \Pr \left[ |\widehat \phi(i) - \phi(i)| \ge \eps | \left| |\Sowni| - \frac{m}{2} \right| > \Delta \right] \Pr \left[ \left| |\Sowni| - \frac{m}{2} \right| > \Delta \right]  \\
    &\le \Pr \left[ |\widehat \phi(i) - \phi(i)| \ge \eps | \left| |\Sowni| - \frac{m}{2} \right| \le \Delta \right] 
    + 2 \exp \left(- \frac{2\Delta^2}{m}\right) \\
    &\le \Pr \left[ |\widetilde{\phi}(i) - \phi(i)| \ge \eps-\frac{4\Delta}{m} | \left| |\Sowni| - \frac{m}{2} \right| \le \Delta \right] 
    + 2 \exp \left(- \frac{2\Delta^2}{m}\right) \\
    &\le \frac{ \Pr \left[ |\widetilde{\phi}(i) - \phi(i)| \ge \eps-\frac{4\Delta}{m} \right] }{ 1 - 2 \exp \left(- \frac{2\Delta^2}{m}\right) } + 2 \exp \left(- \frac{2\Delta^2}{m}\right) \\
    &\le \frac{ 2 \exp \left( -\frac{1}{8} m (\eps - \frac{4\Delta}{m})^2 \right) }{ 1 - 2 \exp \left(- \frac{2\Delta^2}{m}\right) } + 2 \exp \left(- \frac{2\Delta^2}{m}\right) \\
    &\le 3 \exp \left( -\frac{1}{8} m \left(\eps - \frac{4\Delta}{m}\right)^2 \right) 
    + 2 \exp \left(- \frac{2\Delta^2}{m}\right)
\end{align*}
where the last inequality holds whenever $1 - 2 \exp \left(- \frac{2\Delta^2}{m}\right) \ge \frac{2}{3}$. 

We can then optimize this bound by setting $-\frac{1}{8} m \left(\eps - \frac{4\Delta}{m}\right)^2 = - \frac{2\Delta^2}{m}$, where we obtain $\Delta = \frac{m\eps}{8}$, and the bound becomes 
\begin{align*}
    \Pr \left[ |\widehat \phi(i) - \phi(i)| \ge \eps \right] 
    &\le 
    3 \exp \left( -\frac{1}{8} m \left(\eps - \frac{4\Delta}{m}\right)^2 \right) 
    + 2 \exp \left(- \frac{2\Delta^2}{m}\right) \\
    &= 5 \exp \left( - \frac{m\eps^2}{32} \right)
\end{align*}

By union bound, we have 
\begin{align*}
    \Pr_{\widehat \phi_\gt} \left[ \norm{\widehat \phi_\gt - \phi}_2 \ge \eps \right] \le 5n\exp \left( - \frac{m\eps^2}{32n} \right)
\end{align*}
and 
\begin{align*}
    \Pr_{\widehat \phi_\gt} \left[ \norm{\widehat \phi_\gt - \phi}_\infty \ge \eps \right] \le 5n\exp \left( - \frac{m\eps^2}{32} \right)
\end{align*}
By setting $\delta \le 5n\exp \left( - \frac{m\eps^2}{32n} \right)$, we obtain the sample complexity $O\left( \frac{n}{\eps^2} \log(\frac{n}{\delta}) \right)$ for $\ell_2$ norm, and by setting $\delta \le 5n\exp \left( - \frac{m\eps^2}{32} \right)$, we obtain the sample complexity $O\left( \frac{1}{\eps^2} \log(\frac{n}{\delta}) \right)$ for $\ell_\infty$-norm. 
\end{proof}

\subsubsection{Lower Bound for the Banzhaf Value Estimator}
\label{appendix:lower-bound}

\begin{customthm}{\ref{thm:lowerbound}}[restated]
Every (possibly randomized) Banzhaf value estimation algorithm that achieves $(\eps, \delta)$-approximation in $\ell_\infty$-norm for constant $\delta \in (0, 1/2)$ has sample complexity at least $\Omega(\frac{1}{\eps})$. 
\end{customthm}
\begin{proof}
To show the lower bound of the sample complexity for Banzhaf value estimation, we use \emph{Yao's minimax principle}: to show a lower bound on a randomized algorithm, it suffices to define a distribution on some family of instances and show a lower bound for deterministic algorithms on this distribution. 

Fix $\eps \in (0, 1)$. We define the collection of instance $\I_0$ as all utility functions $U$ such that $U(S)=U(S \cup n)$ for all $S \subseteq [n-1]$.  
We define the collection of instance $\I_1$ as all utility functions $U$ s.t. there are exactly $2^{n-1} (2\eps)$ of the $S \subseteq [n-1]$ has $U(S)=0, U(S \cup n)=1$, and for all other $S$ we have $U(S)=U(S \cup n)$. 
We define a distribution over $\I_0 \cup \I_1$ by first randomly picking $\I_0$ or $\I_1$ with probability 1/2, and then picking a utility function from the selected instance class uniformly at random. 

For any $U \in \I_0$, we have $\phi(n; U)=0$, and for any $U \in \I_1$, we have $\phi(n; U) = \frac{2^{n-1} (2\eps)}{2^{n-1}} = 2\eps$. 
Thus, in order to achieve $\norm{\widehat\phi - \phi}_\infty < \eps$, the estimator must be able to distinguish between whether the utility function is from $\I_0$ or $\I_1$. 
For this, it needs to identify at least one $S \subseteq [n-1]$ s.t. $U(S)=0, U(S \cup n)=1$. 
However, since those $S$ are chosen uniformly at random, no matter what sampling strategy the algorithm has, each query succeeds with probability at most $2^{n-1} (2\eps) / 2^{n-1} = 2\eps$. 
Thus, for $m$ queries, the total failure probability is at least $(1-2m\eps)/2$. 
To make the failure probability at most $\delta$, we need number of samples $m$ s.t. $(1-2m\eps)/2 \le \delta$, which leads to the lower bound $m \ge \frac{1-2\delta}{2\eps}$. Thus, we have $m \in \Omega(\frac{1}{\eps})$. 
\end{proof}

\newpage

\subsection{MSR Estimator does not Extend to the Shapley Value and Other Known Semivalues} 
\label{appendix:extend-msr}

In this section, we provide proof and more discussion about why the existence of the MSR estimator is a unique advantage of Banzhaf value. 

\paragraph{Numerical Instability. }
It is easy to see that semivalue can be written as the expectation of \emph{weighted} marginal contribution
\begin{align}
\phi_{\semi}\left(i ; U, N, w\right) 
&:= \frac{1}{n} \sum_{k=1}^{n} w(k) \sum_{S \subseteq N \setminus i, |S|=k-1} \left[ U(S \cup i) - U(S) \right] \nonumber \\
&= \E_{S \sim \unif(2^{N \setminus i})} \left[
\frac{2^{n-1}w(|S|+1)}{n} \left(  U(S \cup i) - U(S) \right) \right] \nonumber  \\
&= \E_{S \sim \unif(2^{N \setminus i})} \left[ \frac{2^{n-1}w(|S|+1)}{n}  U(S \cup i) \right] - \E_{S \sim \unif(2^{N \setminus i})} \left[\frac{2^{n-1}w(|S|+1)}{n}U(S)\right] \nonumber  \\
&= \E_{S \sim \unif( \{S \in 2^{N}: S \own i \} ) } \left[ \frac{2^{n-1}w(|S|)}{n}  U(S) \right] - \E_{S \sim \unif(2^{N \setminus i})} \left[\frac{2^{n-1}w(|S|+1)}{n}U(S)\right] \label{eq:conditional}
\end{align}

Hence, a straightforward way to design MSR estimator for arbitrary semivalue is to to sample $\mS = \{ S_1, \ldots, S_m \}$ each i.i.d. drawn from $\unif(2^{N})$, and estimate $\phi(i)$ as
\begin{align*}
    \widehat \phi_\gt(i) = \frac{1}{|\Sowni|} \sum_{S \in \Sowni} \frac{2^{n-1}w(|S|)}{n} U(S) - \frac{1}{|\Snotowni|} \sum_{S \in \Snotowni} \frac{2^{n-1}w(|S|+1)}{n} U(S)
\end{align*}

For Shapley value, $w(|S|) = {n-1 \choose |S|-1}^{-1}$, which makes the MSR estimator for Shapley value becomes 
\begin{align*}
    \widehat \phi_\gt(i) = \frac{1}{|\Sowni|} \sum_{S \in \Sowni} \frac{1}{n} {n-1 \choose |S|-1}^{-1} U(S) - \frac{1}{|\Snotowni|} \sum_{S \in \Snotowni} \frac{1}{n} {n-1 \choose |S|}^{-1} U(S)
\end{align*}
which is numerically unstable for large $n$ due to the combinatorial coefficients. 

\paragraph{Impossible to construct a special sampling distribution for MSR.} 
The MSR estimator for Banzhaf value samples from $\unif(2^N)$ since for a random set $S \sim \unif(2^N)$, we have its conditional distribution $S | S \notown i \sim \unif(2^{N \setminus i})$ and $S | S \own i \sim \unif( \{S \in 2^{N}: S \own i \} )$, which exactly matches the two distributions the expectation in (\ref{eq:conditional}) is taken over for Banzhaf value. 
For a semivalue with weight function $w$, can we design a similar distribution $\D$ over $2^N$ so that 
$
\phi_{\semi}\left(i ; U, N, w\right) = \E_{S \sim \D|\D \own i} \left[ U(S) \right] - \E_{S \sim \D|\D \notown i} \left[ U(S) \right]
$? The answer is unfortunately negative. 
Note that in order to write $\phi_{\semi}\left(i ; U, N, w\right)$ in this way, we must have 
\begin{align*}
    &\Pr[ \D =  S | i \in S ] = \frac{1}{n} w(|S|+1) \\
    &\Pr[ \D =  S | i \notin S ] = \frac{1}{n} w(|S|) 
\end{align*}
for any $i \in N$. 
Now, we consider a particular $S$ s.t. $i \notin S, j \in S$. Denote $x = \Pr[i \notin \D]$, $y = \Pr[j \notin \D]$, and $k = |S|+1$. 
By Bayes theorem, we have
\begin{align*}
    \Pr[ \D = S | i \notin S, j \in S ]
    &= \frac{ \Pr[j \in \D | \D = S, i \notin \D] \Pr[\D = S|i \notin \D] }{ \Pr[j \in \D] } \\
    &= \frac{w(k-1)}{n(1-y)} \\
    &= \frac{ \Pr[i \notin \D | \D = S, j \in \D] \Pr[\D = S|j \in \D] }{ \Pr[i \notin \D] } \\
    &= \frac{w(k)}{nx}
\end{align*}
Thus we have $w(k-1) x = w(k)(1-y)$. 
Similarly, consider a $S'$ of the same size s.t. $i \in S, j \notin S$, we obtain $w(k-1) y = w(k)(1-x)$. 

Given 
\begin{align*}
    w(k-1) x &= w(k)(1-y) \\
    w(k-1) y &= w(k)(1-x)
\end{align*}
If $x=y$, then we have $x = \frac{w(k-1)}{w(k-1)+w(k)}$ which clearly depends on $k$ unless $w(1), w(2), \ldots, w(n)$ is a geometric series ($w(k-1)+w(k)$ cannot be 0 for all $k$, and $x$ can also not be 0). The only known semivalue-based data valuation method that satisfies this property is the Banzhaf value.

If $x \ne y$, then we have $w(k-1) (x-y) = w(k)(x-y)$, which clearly leads to $w(k-1)=w(k)$ where Banzhaf value is still the only choice.

\newpage

\subsection{Robustness of MSR Estimator}
\label{appendix:msr-robustness}

\textbf{Robustness of MSR Estimator Under Noisy Utility Function.} As discussed in Section~\ref{sec:utility_stochastic}, the utility function $U$ is re-defined as the expected model performance due to the stochasticity of the underlying learning algorithm. 
Hence, the actual estimator of the Banzhaf value that we build is based upon the noisy variant $\Uhat$: 
\begin{align}
    \tildephi_\gt(i) = \frac{1}{|\Sowni|} \sum_{S \in \Sowni} \Uhat(S) - \frac{1}{|\Snotowni|} \sum_{S \in \Snotowni} \Uhat(S).
\end{align}

Hence, it is interesting to understand the impact of noisy utility function evaluation on the sample complexity of the MSR estimator.


\begin{theorem}
When $\|U - \Uhat\|_2 \le \gamma$, $\tildephi_\gt$ is $(\eps + \frac{\gamma \sqrt{n}}{2^{n/2-1}}, \delta)$-approximation in $\ell_2$-norm with $O(\frac{n}{\eps^2} \log(\frac{n}{\delta}))$ calls, and $(\eps + \frac{\gamma}{2^{n/2-1}}, \delta)$-approximation in $\ell_\infty$-norm with $O(\frac{1}{\eps^2} \log(\frac{n}{\delta}))$ calls to $\Uhat$. 
\label{thm:gt-under-noisy-utility}
\end{theorem}

\add{The theorem above shows that our MSR algorithm has the same sample complexity in the presence of noise in $\Uhat$, with a small extra irreducible error since typically $\gamma \propto \sqrt{2^n}$. }

Recall that
\begin{align*}
    \tildephi_\gt(i) = \frac{1}{|\Sowni|} \sum_{S \in \Sowni} \Uhat(S) - \frac{1}{|\Snotowni|} \sum_{S \in \Snotowni} \Uhat(S) 
\end{align*}

\begin{customthm}{\ref{thm:gt-under-noisy-utility}}
When $\|U - \Uhat\|_2 \le \gamma$, $\tildephi_\gt$ is $(\eps + \frac{\gamma \sqrt{n}}{2^{n/2-1}}, \delta)$-approximation in $\ell_2$-norm with $O(\frac{n}{\eps^2} \log(\frac{n}{\delta}))$ calls, and $(\eps + \frac{\gamma}{2^{n/2-1}}, \delta)$-approximation in $\ell_\infty$-norm with $O(\frac{1}{\eps^2} \log(\frac{n}{\delta}))$ calls to $\Uhat$. 
\end{customthm}
\begin{proof}

Note that for each $i \in N$, 
\begin{align*}
    \left| \tildephi(i) - \phi(i) \right|
    \le \left| \tildephi(i) - \widehat \phi(i) \right|
    + \left| \widehat \phi(i) - \phi(i) \right|
\end{align*}
From Theorem \ref{thm:gt}, we have 
\begin{align*}
    \Pr_{\widehat \phi} \left[\left| \widehat \phi(i) - \phi(i) \right| \ge \Delta \right]
    \le 5 \exp\left(-\frac{m}{32}\Delta^2\right)
\end{align*}
Now we bound $\left| \tildephi(i) - \widehat \phi(i) \right|$. 
\begin{align*}
    \left| \tildephi(i) - \widehat \phi(i) \right| 
    &= \left| \frac{1}{|\Sowni|} \sum_{S \in \Sowni} \left(U(S)-\Uhat(S)\right) - \frac{1}{|\Snotowni|} \sum_{S \in \Snotowni} \left(U(S)-\Uhat(S)\right) \right| \\
    &\le \frac{1}{|\Sowni|} \sum_{S \in \Sowni} \left|U(S)-\Uhat(S)\right| 
    + \frac{1}{|\Snotowni|} \sum_{S \in \Snotowni} \left|U(S)-\Uhat(S)\right|
\end{align*}

Given $\norm{U - \Uhat} \le \gamma$, we bound $\left| \tildephi(i) - \widehat \phi(i) \right|$ as follows: 
\begin{align}
    &\Pr \left[ \left| \tildephi(i) - \widehat \phi(i) \right| \ge \eps \right] \nonumber \\
    &\le \Pr \left[ \frac{1}{|\Sowni|} \sum_{S \in \Sowni} \left|U(S)-\Uhat(S)\right| 
    + \frac{1}{|\Snotowni|} \sum_{S \in \Snotowni} \left|U(S)-\Uhat(S)\right| \ge \eps \right] \nonumber \\
    &\le \Pr \left[ \frac{1}{|\Sowni|} \sum_{S \in \Sowni} \left|U(S)-\Uhat(S)\right| 
    + \frac{1}{|\Snotowni|} \sum_{S \in \Snotowni} \left|U(S)-\Uhat(S)\right| \ge \eps | \left| |\Sowni| - \frac{m}{2} \right| \le \Delta \right] + 2 \exp \left(- \frac{2\Delta^2}{m}\right) \nonumber \\
    &\le \frac{ 
    \Pr \left[ \frac{2}{m} \sum_{S \subseteq \mS} \left| U(S) - \Uhat(S) \right| \ge \eps - \frac{4\Delta}{m} \right]
    }{1 - 2 \exp \left(- \frac{2\Delta^2}{m}\right) } + 2 \exp \left(- \frac{2\Delta^2}{m}\right) 
    \label{eq:cont}
\end{align}

By $\norm{U - \Uhat} \le \gamma$, we have 
\begin{align*}
    \E[|U(S)-\widehat{U}(S)|] = 
    \frac{1}{2^n} \sum_{S \subseteq N} \left|U(S) - \Uhat(S)\right| = \frac{1}{2^n} \norm{U - \Uhat}_1 
    \le \frac{\sqrt{2^{n}}}{2^n} \norm{U - \Uhat}
    = \frac{\gamma}{2^{n/2}}
\end{align*}

Set $\eps' = \eps - \frac{\gamma}{2^{n/2-1}}$.
Thus 
\begin{align*}
    (\ref{eq:cont}) &= 
    \frac{ 
    \Pr \left[ \frac{2}{m} \sum_{S \subseteq \mS} \left| U(S) - \Uhat(S) \right| - \frac{\gamma}{2^{n/2-1}} \ge \eps' - \frac{4\Delta}{m} \right]
    }{1 - 2 \exp \left(- \frac{2\Delta^2}{m}\right) } + 2 \exp \left(- \frac{2\Delta^2}{m}\right)  \\
    &\le \frac{ 
    \Pr \left[ \frac{2}{m} \sum_{S \subseteq \mS} \left| U(S) - \Uhat(S) \right| - 2 \E[|U(S)-\widehat{U}(S)|] \ge \eps' - \frac{4\Delta}{m} \right]
    }{1 - 2 \exp \left(- \frac{2\Delta^2}{m}\right) } + 2 \exp \left(- \frac{2\Delta^2}{m}\right)  \\
    &\le \frac{ 
    \Pr \left[ \left| \frac{2}{m} \sum_{S \subseteq \mS} \left| U(S) - \Uhat(S) \right| - \frac{\gamma}{2^{n/2-1}} \right| \ge \eps' - \frac{4\Delta}{m} \right]
    }{1 - 2 \exp \left(- \frac{2\Delta^2}{m}\right) } + 2 \exp \left(- \frac{2\Delta^2}{m}\right) \\
    &\le \frac{ 
    2 \exp \left( -\frac{1}{8} m \left(\eps' - \frac{4\Delta}{m}\right)^2 \right)
    }{1 - 2 \exp \left(- \frac{2\Delta^2}{m}\right) } + 2 \exp \left(- \frac{2\Delta^2}{m}\right) \\
    &\le 5 \exp \left(-\frac{m}{32} (\eps')^2\right)
\end{align*}

Thus, we have 
\begin{align*}
    \Pr_{\widehat \phi} \left[\left| \widehat \phi(i) - \tildephi(i) \right| \ge \eps + \frac{\gamma}{2^{n/2-1}} \right]
    \le 5 \exp\left(-\frac{m}{32}\eps^2\right)
\end{align*}

Therefore, 
\begin{align*}
    \Pr \left[ \left| \tildephi(i) - \phi(i) \right| \ge \eps + \frac{\gamma}{2^{n/2-1}} \right] 
    &\le \Pr \left[ \left| \tildephi(i) - \widehat \phi(i) \right| + \left| \widehat \phi(i) - \phi(i) \right| \ge \eps + \frac{\gamma}{2^{n/2-1}} \right] \\
    &= 1 - \Pr \left[ \left| \tildephi(i) - \widehat \phi(i) \right| + \left| \widehat \phi(i) - \phi(i) \right| \le \eps + \frac{\gamma}{2^{n/2-1}} \right] \\
    &\le 1 - \Pr \left[ \left| \tildephi(i) - \widehat \phi(i) \right| \le \frac{\eps}{2} + \frac{\gamma}{2^{n/2-1}} \land 
    \left| \widehat \phi(i) - \phi(i) \right| \le \frac{\eps}{2} \right] \\
    &\le \Pr \left[ \left| \tildephi(i) - \widehat \phi(i) \right| \ge \frac{\eps}{2} + \frac{\gamma}{2^{n/2-1}} \right] + 
    \Pr \left[\left| \widehat \phi(i) - \phi(i) \right| \ge \frac{\eps}{2} \right] \\
    &= 10 \exp \left( - \frac{m}{128} \eps^2 \right)
\end{align*}

By union bound, we have 
\begin{align*}
    \Pr \left[ \norm{\tildephi_\gt - \phi}_2 \ge \eps + \frac{\gamma \sqrt{n}}{2^{n/2-1}} \right] \le 10 n \exp \left( - \frac{m\eps^2}{128n} \right)
\end{align*}
and 
\begin{align*}
    \Pr \left[ \norm{\tildephi_\gt - \phi}_\infty \ge \eps + \frac{\gamma}{2^{n/2-1}} \right] \le 10 n \exp \left( - \frac{m\eps^2}{128} \right)
\end{align*}

By setting $\delta \le  10 n \exp \left( - \frac{m\eps^2}{128n} \right)$, we obtain the sample complexity $O\left( \frac{n}{\eps^2} \log(\frac{n}{\delta}) \right)$ for $\ell_2$-norm, 
and by setting $\delta \le  10 n \exp \left( - \frac{m\eps^2}{128} \right)$, we obtain the sample complexity $O\left( \frac{1}{\eps^2} \log(\frac{n}{\delta}) \right)$ for $\ell_\infty$-norm. 
\end{proof}

\newpage 

\subsection{Stability of Banzhaf value in $\ell_2$-norm}
\label{sec:l2-robustness}

As mentioned previously, we can alternatively view a semivalue as a function $\phi: \R^{2^n} \rightarrow \R^n$ which takes a utility function $U \in \R^{2^n}$ as input, and output the values of data points $\phi(U) \in \R^n$. By taking this functional view, a natural robustness measure for semivalue $\phi(\cdot; w)$ is its Lipschitz constant $L$, which is defined as the smallest constant such that 
\begin{align*}
    \norm{\phi(U; w) - \phi(\widehat U; w)} 
    \le 
    L \norm{U - \widehat U}
\end{align*}
for all possible pairs of $U$ and $\widehat U$. 

\begin{theorem}
Among all semivalues, Banzhaf value ($w(k) = \frac{n}{2^{n-1}}$) achieves the smallest Lipschitz constant $L = \frac{1}{2^{n/2-1}}$. In other words, for the Banzhaf value we have
\begin{align*}
    \norm{\phi_{\banz}(U) - \phi_{\banz}(\widehat U; w)} \le \frac{1}{2^{n/2-1}} \norm{U - \widehat U}
\end{align*}
for all possible pairs of $U$ and $\widehat U$, and $L=\frac{1}{2^{n/2-1}}$ is the smallest constant among all semivalues. 
\end{theorem}

\begin{proof}

Recall that a semivalue has the following representation
\begin{align*}
\phi_{\semi}\left(i ; U, w\right) 
&:= \frac{1}{n} \sum_{k=1}^{n} w(k) \sum_{S \subseteq N \setminus \{i\}, |S|=k-1} \left[ U(S \cup i) - U(S) \right]
\end{align*}

An interesting observation about semivalue is that the transformation $\phi: \R^{2^n} \rightarrow \R^n$ is always a linear transformation. 
Thus, for every semivalue, we can define \emph{Semivalue matrix} $S_n \in \R^{n \times 2^n}$ where $\phi(U) = S_n U$. 
We denote the $i$th row of $S_n$ as $(S_n)_i$, and the entry in the $i$th row corresponding to subset $S$ as $(S_n)_{i, S}$. 
It is not hard to see that 
\begin{align*}
    (S_n)_{i, S} &= \frac{1}{n} w(|S|) \text{  if } i \in S \\
    (S_n)_{i, S} &= -\frac{1}{n} w(|S|+1) \text{  if } i \notin S
\end{align*}

The Lipschitz constant of $\phi$ is thus equal to the operator norm of matrix $S_n$, which is the square root of the largest eigenvalue of matrix $S_n S_n^T$. Now we compute the eigenvalue of matrix $S_n S_n^T$.

For matrix $S_n S_n^T$, its diagonal entry is
\begin{align*}
    d_1 &= \sum_{S \in 2^n, i \in S} \frac{1}{n^2} w^2(|S|) +  \sum_{S \in 2^n, i \notin S} \frac{1}{n^2} w^2(|S|+1) \\
    &= \frac{1}{n^2} \left[ \sum_{k=1}^n {n-1 \choose k-1} w^2(k) + \sum_{k=1}^n {n-1 \choose k-1} w^2(k) \right] \\
    &= \frac{2}{n^2} \sum_{k=1}^n {n-1 \choose k-1} w^2(k)
\end{align*}

and its non-diagonal entry is 
\begin{align*}
    d_2 &= 
    \sum_{k=2}^n {n-2 \choose k-2} \left(\frac{1}{n} w(k) \right)^2 + 2 \sum_{k=1}^{n-1} {n-2 \choose k-1} \left(-\frac{1}{n^2} w(k) w(k+1) \right) +
    \sum_{k=0}^{n-2} {n-2 \choose k} \left(\frac{1}{n} w(k+1) \right)^2 \\
    &= \frac{1}{n^2} \sum_{k=0}^{n-2} {n-2 \choose k} \left[ w^2(k+2) - 2w(k+1)w(k+2) + w^2(k+1) \right] \\
    &= \frac{1}{n^2} \sum_{k=0}^{n-2} {n-2 \choose k} 
    \left( w(k+2) - w(k+1) \right)^2
\end{align*}

Therefore we can write $S_n S_n^T = (d_1 - d_2)\ind_n + d_2 \mathbf{1}_n$ where $\ind_n \in \R^{n \times n}$ is the identity matrix and $\mathbf{1}_n \in \R^n$ is all-one matrix. 
The two eigenvalues are $d_1 + (n-1)d_2$ and $d_1 - d_2$. 
Since $d_2 \ge 0$, the top eigenvalue is $d_1 + (n-1)d_2$. Therefore, our goal is to find weight function $w$ such that
\begin{align*}
    &\min_w d_1 + (n-1)d_2 \\
    &\text{subject to } \sum_{k=1}^{n}{n-1 \choose k-1} w(k)=n
\end{align*}
that is, we want to solve
\begin{align*}
    &\min_w  \frac{2}{n^2} \sum_{k=1}^n {n-1 \choose k-1} w(k)^2 + \frac{n-1}{n^2} \sum_{k=1}^{n-1} {n-2 \choose k-1} \left( w(k) - w(k+1) \right)^2  \\
    &\text{subject to } \sum_{k=1}^{n}{n-1 \choose k-1} w(k)=n
\end{align*}
Note that by Cauchy-Schwarz inequality, 
\begin{align*}
    n^2 
    &= \left( \sum_{k=1}^{n}{n-1 \choose k-1} w(k) \right)^2 \\
    &= \left( \sum_{k=1}^{n}\sqrt{ {n-1 \choose k-1} } \sqrt{ {n-1 \choose k-1} }  w(k) \right)^2 \\
    &\le \left( \sum_{k=1}^{n}{n-1 \choose k-1} \right)
    \left( \sum_{k=1}^{n}{n-1 \choose k-1} w(k)^2 \right) \\
    &= 2^{n-1} \sum_{k=1}^{n}{n-1 \choose k-1} w(k)^2
\end{align*}
Thus the first term in the objective function is lower bounded by $\frac{2}{2^{n-1}}$, which is achieved when $w(1) = \ldots = w(n) = \frac{n}{2^{n-1}}$, and this is also where the minimum of the second term achieved, i.e., just 0. 
Thus, the minimum possible $\min_w = d_1 + (n-1)d_2 = \frac{1}{2^{n-2}}$, and thus the operator norm of $S_n = \sqrt{\frac{1}{2^{n-2}}} = \frac{1}{2^{n/2-1}}$. 
\end{proof}

\newpage

\section{ Experiment Settings \& Additional Experimental Results }
\label{appendix:experiment}

We provide a summary of the content in this section for the convenience of the readers. 
\begin{itemize}
    \item Appendix \ref{appendix:experiment-figuretwo}: Experiment Settings for Figure \ref{fig:checkvar-value} and \ref{fig:topk} in the main text.
    \item Appendix \ref{appendix:settings-efficiency}: Experiment Settings for Sample Efficiency Experiment in Section \ref{sec:eval-efficiency}. 
    \item Appendix \ref{appendix:settings-app}: Experiment Settings for Applications in Section \ref{sec:eval-applications}. 
    \item Appendix \ref{appendix:additional-rank-stab}: Experiment Settings and Additional Results for Rank Stability Experiment in Section \ref{sec:eval-ranking}.
    \item Appendix \ref{appendix:settings-ranking}: Additional Results for Rank Stability Experiment on Tiny Datasets. 
    \item Appendix \ref{appendix:randomized-smoothing}: Additional Results for Rank Stability on Gradient Descent with Randomized Smoothing. 
\end{itemize}

\subsection{ Experiment Settings for Figure \ref{fig:checkvar-value} and \ref{fig:topk} in the Main Text. }
\label{appendix:experiment-figuretwo}

We estimate the LOO error, Shapley, and Banzhaf value on a size-2000 CIFAR10 dataset where we randomly flip the label of 10\% data points. We use the state-of-the-art estimator for the Shapley and Banzhaf value for Figure \ref{fig:checkvar-value} (b) and (c), i.e., Permutation sampling for the Shapley value and our MSR estimator for Banzhaf value. For both the Shapley and Banzhaf value, we set the number of samples as 50,000 where the sampling distributions depend on the corresponding estimators. 
We compute the LOO with its exact formula but with noisy utility scores, and we align the number of (potentially repeated) samples also as 50,000. For each sample $S$, we train 5 models on it with different random seeds, and obtain 5 noisy versions of $U(S)$. We then compute 5 different LOO/Shapley/Banzhaf values for 20 randomly selected CIFAR10 images (including 5 mislabeled images), and draw the corresponding box-plot in Figure \ref{fig:checkvar-value}. 
The learning architecture we use is a standard CNN adapted from PyTorch tutorial\footnote{\url{https://pytorch.org/tutorials/beginner/blitz/cifar10_tutorial.html}}, with batch size $32$, learning rate $10^{-3}$ and Adam optimizer for training. 

In Figure \ref{fig:topk}, we compare the stability of different data valuation techniques in maintaining a consistent set of top-influence data points.  Specifically, we follow the same protocol as the setting for Figure \ref{fig:checkvar-value} and compute 5 different versions of data value scores for each data point. 
We then count the percentage of data points that are consistently ranked in the top or bottom-$k\%$ across all the runs. 

\begin{remark}
    The results in Figure \ref{fig:checkvar-value} and \ref{fig:topk} in the main text also depend on the robustness of data valuation methods' corresponding estimator. We use the best-known estimators for each data valuation method in the experiment to reduce the impact of the estimation procedure as much as possible.
\end{remark}

\subsection{Experiment Settings for Sample Efficiency Experiment in Section \ref{sec:eval-efficiency}}
\label{appendix:settings-efficiency}

For Figure \ref{fig:convergence} (a), we use a synthetic dataset with only 10 data points. To generate the synthetic dataset, we sample 10 data points from a bivariate Gaussian distribution where the means are $0.1$ and $-0.1$ on each dimension, and the covariance matrix is the identity matrix. The labels are assigned to be the sign of the sum of the two features. The utility of a subset is the test accuracy of the model trained on the subset. A logistic regression classifier trained on the 10 data points achieves around 80\% test accuracy. 
We show the Banzhaf value estimate for \emph{one} data point in Figure \ref{fig:convergence} (a). 

For Figure \ref{fig:convergence} (b), we use a size-500 MNIST dataset. The Relative Spearman Index is computed by the Spearman Index of the ranking of the value estimates between the current iteration, and the ranking of the value estimates when given additional 1000 samples. 
The learning architecture we use is LeNet \citep{lecun1989handwritten}, with batch size $32$, (initial) learning rate $10^{-3}$ and Adam optimizer for training.

\subsection{Experiment Settings for Applications in Section \ref{sec:eval-applications}}
\label{appendix:settings-app}

\begin{table}[t]
\centering
\begin{tabular}{@{}cc@{}}
\toprule
\textbf{Dataset} & \textbf{Source}                        \\ \midrule
MNIST            & \citet{lecun1998mnist}                \\
FMNIST           & \citet{xiao2017fashion}               \\
CIFAR10          & \citet{krizhevsky2009learning}        \\
Click            & \url{https://www.openml.org/d/1218}  \\
Fraud            & \citet{dal2015calibrating}            \\
Creditcard       & \citet{yeh2009comparisons}            \\
Vehicle          & \citet{duarte2004vehicle}             \\
Apsfail          & \url{https://www.openml.org/d/41138} \\
Phoneme          & \url{https://www.openml.org/d/1489}  \\
Wind             & \url{https://www.openml.org/d/847}   \\
Pol              & \url{https://www.openml.org/d/722}   \\
CPU              & \url{https://www.openml.org/d/761}   \\
2DPlanes         & \url{https://www.openml.org/d/727}   
\\ \bottomrule
\end{tabular}
\caption{A summary of datasets used in Section \ref{sec:eval-applications}'s experiments.}
\label{tb:datasets}
\end{table}

\subsubsection{Datasets \& Models}
\label{appendix:settings-dataset}

A comprehensive list of datasets and sources is summarized in Table \ref{tb:datasets}. 
Similar to the existing data valuation literature \citep{ghorbani2019data, kwon2021beta, jia2019towards, wang2021improving}, we preprocess datasets for the ease of training. 
For Fraud, Creditcard, Vehicle, and all datasets from OpenML, we subsample the dataset to balance positive and negative labels. For these datasets, if they have multi-class, we binarize the label by considering $\ind[y=1]$.  
For the image dataset CIFAR10, we follow the common procedure in prior works \citep{ghorbani2019data, jia2019towards, kwon2021beta}: we extract the penultimate layer outputs from the pre-trained ResNet18 \citep{he2016deep}. The pre-training is done with the ImageNet dataset \citep{deng2009imagenet} and the weight is publicly available from PyTorch. We choose features from the class of Dog and Cat. The extracted outputs have dimension 512. For the image dataset MNIST and FMNIST, we directly train on the original data format, which is a more challenging setting compared with the previous literature. 

For MNIST and FMNIST, we use LeNet \citep{lecun1989handwritten}, with batch size $128$, (initial) learning rate $10^{-3}$ and Adam optimizer for training. For CIFAR10 dataset, we use a two-layer MLP where there are 256 neurons in the hidden layer, with activation function ReLU, with batch size $128$, (initial) learning rate $10^{-3}$ and Adam optimizer for training. For the rest of the datasets, we use a two-layer MLP where there are 100 neurons in the hidden layer, with activation function ReLU, (initial) learning rate $10^{-2}$, and Adam optimizer for training. 
We use batch size $128$ for Creditcard, Apsfail, Click, and CPU dataset, and batch size $32$ for the rest of datasets. 


\subsubsection{Experiment Settings}
\label{appendix:settings-hyper}

For MNIST, FMNIST, and CIFAR10, we consider the number of data points being valued as 2000. 
For Click dataset, we consider the number of data points to be valued as 1000. 
For Phoneme dataset, we consider the number of data points to be valued as 500. 
For the rest of the datasets, we consider the number of data points to be valued as 200. 
For each data value we show in Table \ref{tb:weighted-sample} and \ref{tb:mislabel-detection}, we use the corresponding state-of-the-art estimator to estimate them (for Data Shapley, we use Permutation Sampling; for Data Banzhaf, we use our MSR estimator; for Beta Shapley, we use the Monte Carlo estimator by \citet{kwon2021beta}). 
We stress that the Monte Carlo estimator by \citet{kwon2021beta} is not numerically stable when the training set size $> 500$, so for datasets with $> 500$ data points (MNIST, FMNIST, CIFAR10, and Click), we omit the results for Beta Shapley. We set the number of samples to estimate Data Banzhaf, Data Shapley, and Beta Shapley as 100,000. 

All of our experiments are performed on Tesla P100-PCIE-16GB GPU.

\paragraph{Learning with Weighted Samples. }
For each estimated data value, we normalize it to $[0, 1]$ by $\frac{\texttt{value}-\texttt{min}}{\texttt{max}-\texttt{min}}$. 
Let $\phi(i)$ be the \emph{normalized} value for data point $i$. We compare the test accuracy of a weighted risk minimizer $f_\phi$ defined as 
\begin{align}
    f_\phi := \argmin_f \sum_{i \in N} \phi(i) \mathrm{loss}_f(i)
\end{align}
where $\mathrm{loss}_f(i)$ denote the loss of $f$ on data point $i \in N$.

\paragraph{Noisy Label Detection. }
We flip 10\% of the labels by picking an alternative label from the rest of the classes uniformly at random.

\begin{figure}[h]
    \centering
    \includegraphics[width=\columnwidth]{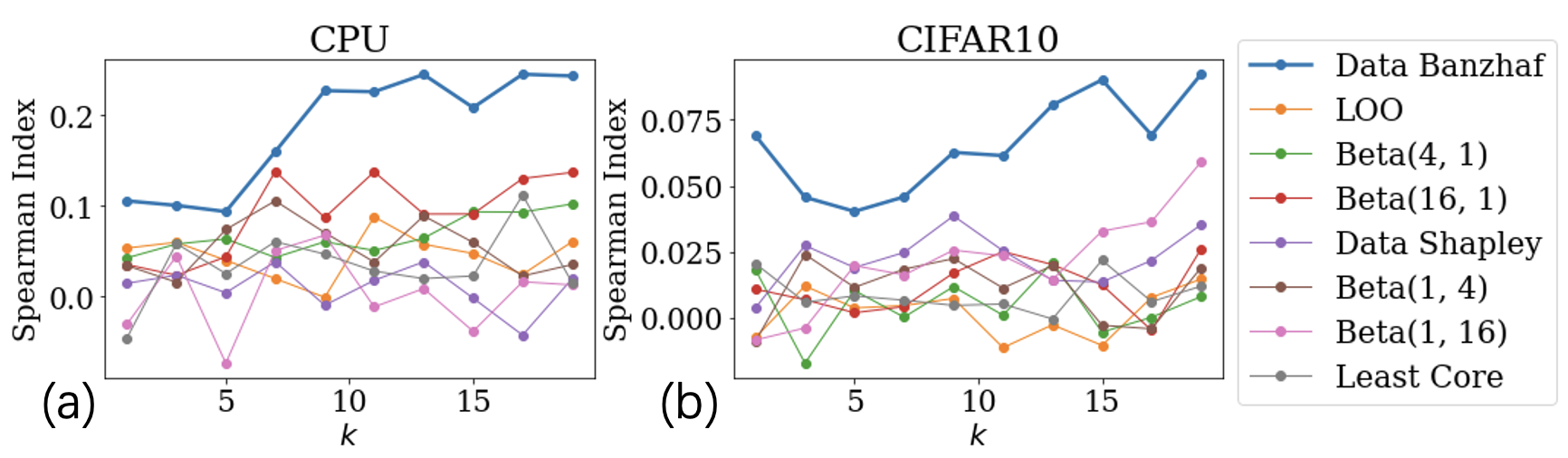}
    \caption{
    The stability of data value ranking measured by Spearman index between the ranking of ``ground-truth'' data value and the ranking of data value estimated from noisy utility scores, where (a) is on CPU dataset and (b) is on CIFAR10 dataset (same as Figure \ref{fig:full} except for the additional curve for the Least core). 
    }
    \label{fig:rank-stab-natural}
\end{figure}

\subsection{Experiment Settings and Additional Results for Rank Stability Experiment in Section \ref{sec:eval-ranking}}
\label{appendix:additional-rank-stab}


In this section, we describe the detailed settings and additional results on comparing the rank stability of different data values on natural datasets. We also evaluate an additional baseline, the \emph{least cores}, another existing data value notion which is \emph{not} a semivalue but also originates from cooperative game theory (see the description in Appendix \ref{appendix:related-work}). The estimation algorithm for the least core is the Monte Carlo algorithm from \citet{yan2020ifyoulike}.  

\paragraph{Settings. }
We experiment on `CPU' (200 data points) and `CIFAR10' (500 data points) datasets from Table \ref{tb:datasets}. 
The data preprocessing procedure, model training hyperparameters, and the estimation algorithm for semivalues are the same as what have described in Appendix \ref{appendix:settings-dataset} and \ref{appendix:settings-hyper}. The perturbations of the model performance scores are caused by the randomness in neural network initialization and mini-batch selection in SGD. 

The tricky part of the experiment design is that we need to find a way to adjust the scale of the perturbation caused by a natural stochastic learning algorithm. 
Our preliminary experiments show that the variance of performance scores does not have a clear dependency on the training hyperparameters such as mini-batch sizes. 
To solve this challenge, we design the following procedure to control the magnitude of the perturbation with a single parameter $k$:
\begin{enumerate}
    \item Sample $m$ data subsets $S_1$, \ldots, $S_{m}$ (the subset sampling distributions depend on specific semivalue estimators). 
    \item For each subset $S_i$, we execute $\widehat{U}(S_i)$ for $k$ times and obtain $k$ independent performance score samples $u_1, \ldots, u_{k} \sim \widehat{U}(S_i)$. We compute $\widetilde{U}_k(S_i) = \frac{1}{k} \sum_{j=1}^{k} u_{j}$. 
    \item Estimate the semivalue with the corresponding estimators based on samples $\widetilde{U}_k(S_1), \ldots, \widetilde{U}_k(S_m)$.
\end{enumerate}
In other words, $k$ is the number of runs that we execute a stochastic learning algorithm in order to estimate the expected utility on a given subset. With a larger $k$, the noise in the estimated utility will become smaller.
Since it is infeasible to compute the ground-truth data values for this experiment, we approximate the ground-truth by setting $k=50$ (called \emph{reference data value} in the caption of figures). 
As $k$ increases, the difference between $\widehat{U}_k(S_i)$ and the approximated ground-truth scores will be smaller with high probability. 
We set the budget of samples used to estimate the semivalues as $m = 2000$ (same for the ground-truth) for all semivalues for a fair comparison. 
It is worth noting that in this case, the rank stability is not just related to the property of the data value notion, but also the corresponding estimator. 

\paragraph{Results. }
We plot the Spearman index between the approximated ground-truth data value ranking and the estimated data value ranking with different $k$s in Figure \ref{fig:rank-stab-natural}, where (b) is the same figure we show in the main text (with the additional baseline of the Least core). As we can see, Data Banzhaf once again outperforms all other data value notions on both datasets. It achieves better rank stability than others by a large margin for a wide range of $k$s.

\subsection{Additional Results for Ranking Stability Experiment on Tiny Datasets}
\label{appendix:settings-ranking}

The ranking stability results in Section \ref{sec:eval-ranking} and Appendix \ref{appendix:additional-rank-stab} do not compute the exact data value but use the corresponding estimation algorithms. In this section, we present an additional result of ranking stability on tiny datasets when we are able to compute the exact data value. 

Similar to the evaluation protocol for sample efficiency comparison, we experiment on a synthetic dataset with a scale ($10$ data points) that we can compute the exact ranking for different data value notions. 
We use the same synthetic dataset as in the sample efficiency experiment in Section \ref{sec:eval-efficiency}. 
The performance score of a subset is the test accuracy of the Logistic regression model trained on the subset. 
\add{In Figure \ref{fig:rank-stab-gaussian}, we plot the Spearman index between the ranking of exact data values and the ranking of data values computed from noisy utility scores.}
\add{For each noise scale $\sigma$ on the x-axis of Figure \ref{fig:rank-stab-gaussian} (a), we add random Gaussian noise $\N(0, \sigma \textbf{I})$ to perturb the performance score. That is, $\widehat{U} = U + \N(0, \sigma \textbf{I})$. We then compute the Spearman index between the ranking of exact data values (derived from $U$) and the ranking of data values derived from noisy utility scores $\widehat{U}$. We repeat this procedure 20 times and take the average Spearman index for each point in the figure. }

\add{
The main considerations behind the design choices of synthetic dataset and Gaussian noise addition are the following:
\begin{packeditemize}
    \item In order to rule out the influence of estimation error, we would like to compute the \emph{exact} ranking of data points in terms of different data value notions, which means that we can only use a toy example with $\le 15$ data points. In this case, it does not make sense to use SGD for training. 
    \item According to our preliminary experiment results, the variance of performance scores does not have a clear dependency on the SGD's training hyperparameters such as mini-batch sizes. The relationship between performance variance and training hyperparameters is an interesting direction for future work. 
\end{packeditemize}
}


\begin{figure}[t]
    \centering
    \includegraphics[width=0.5\columnwidth]{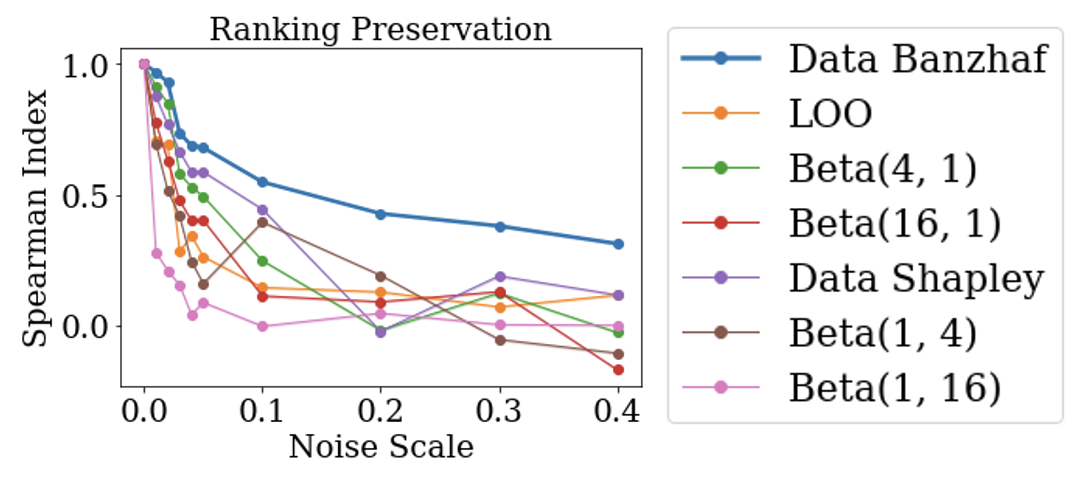}
    \caption{Impact of the noise in utility scores on the stability of data value ranking measured by Spearman index between the ranking of exact data value and the ranking of data value estimated from noisy utility scores.}
    \label{fig:rank-stab-gaussian}
\end{figure}



\subsection{Additional Results for Rank Stability on Gradient Descent with Randomized Smoothing}
\label{appendix:randomized-smoothing}

\newcommand{\loss}{\mathrm{loss}}

In our experiment, we mainly use SGD and its variants as the test case since SGD is arguably the most frequently used stochastic learning algorithm nowadays. However, the robustness guarantee derived in our theory (Section \ref{sec:databanzhaf}) is agnostic to the structure of perturbation, which means that it applies to perturbations caused by arbitrary kinds of learning algorithms. Therefore, we expect to get similar rank stability results when experimenting with other kinds of stochastic learning algorithms. For completeness, we perform an additional ranking stability experiment with another useful stochastic learning algorithm, \emph{gradient descent with randomized smoothing} \citep{duchi2012randomized}.

We use $\loss(\theta, N) = \sum_{i \in N} \loss(\theta, i)$ to denote the loss of model with parameter $\theta$ on the dataset $N$. For regular gradient descent, at iteration $t$, the model is updated as
\begin{align}
    \theta_{t+1} = \theta_t - \eta \g \loss(\theta_t, N)
\end{align}
where $\g$ denotes the derivative with respect to $\theta$. 
In contrast to the regular gradient descent, the randomized smoothing technique convolves Gaussian noise with the original learning loss function and the model is updated instead by ``smoothed gradient'':
\begin{align}
    \theta_{t+1} = \theta_t - \eta \frac{1}{\ell} \sum_{j=1}^\ell 
    \g \loss(\theta_t + \alpha \N(0, \textbf{I}), N)
\end{align}

We compare the rank stability of different semivalues on the CPU dataset and the CIFAR10 dataset, with exactly the same experiment setting as in Appendix \ref{appendix:additional-rank-stab}, except for replacing SGD-based training with the gradient descent with randomized smoothing technique. We set $\ell=1$ to introduce larger randomness, and the smoothing radius $\alpha$ to be equal to the learning rate. 
The results are shown in Figure \ref{fig:randomized-smoothing}. 
As we can see, Data Banzhaf once again outperforms all other data value notions when the sources of performance score perturbation are changed from SGD to gradient descent with randomized smoothing. 

\begin{figure}[t]
    \centering
    \includegraphics[width=\columnwidth]{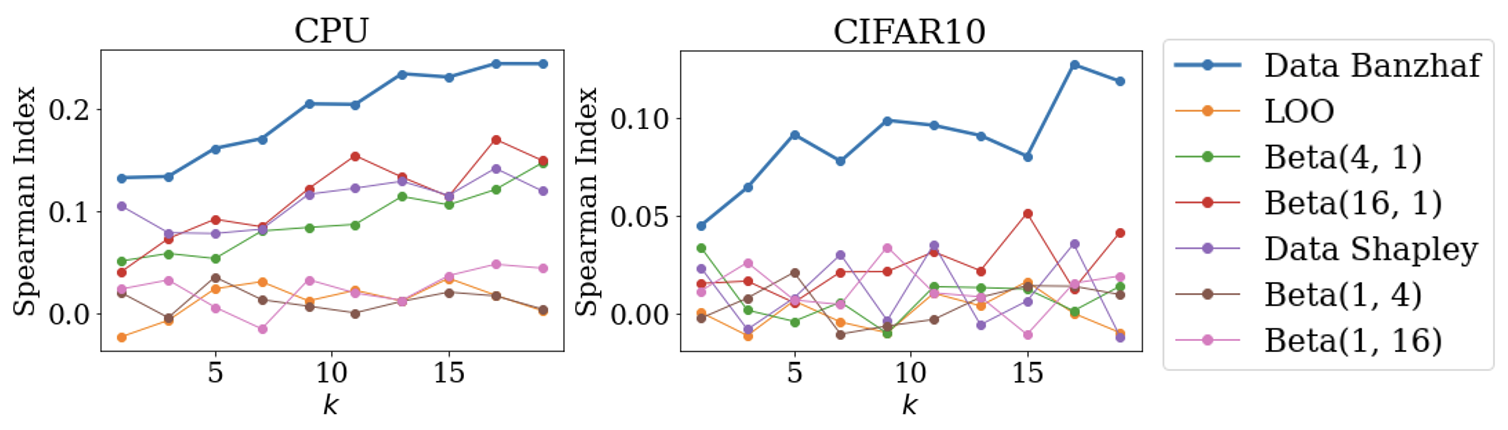}
    \caption{
    The stability of data value ranking measured by Spearman index between the ranking of reference data value and the ranking of data value estimated from noisy utility scores, where (a) is on CPU dataset and (b) is on CIFAR10 dataset. 
    }
    \label{fig:randomized-smoothing}
\end{figure}

\end{document}